\newcount\Comments
\documentclass{article}

\usepackage{color}
\definecolor{darkgreen}{rgb}{0,0.5,0}
\definecolor{purple}{rgb}{1,0,1}
\newcommand{\kibitz}[2]{\ifnum\Comments=1\textcolor{#1}{#2}\fi}

\PassOptionsToPackage{numbers, compress}{natbib}

\usepackage[final]{neurips_2022}

\usepackage[utf8]{inputenc} 
\usepackage[T1]{fontenc}    
\usepackage{hyperref}       
\usepackage{url}            
\usepackage{booktabs}       
\usepackage{amsfonts}       
\usepackage{nicefrac}       
\usepackage{microtype}      
\usepackage{xcolor}         
\usepackage{dirtytalk}

\usepackage{amsmath}
\usepackage{amsthm}
\usepackage{bbm}
\usepackage{amssymb}
\DeclareMathOperator*{\argmax}{arg\,max}
\DeclareMathOperator*{\argmin}{arg\,min}
\DeclareMathOperator*{\E}{\mathbb{E}}
\usepackage{mathtools}

\usepackage[linesnumbered,ruled,vlined]{algorithm2e}

\newtheorem{theorem}{Theorem}
\newtheorem{proposition}[theorem]{Proposition}
\newtheorem{definition}{Definition}
\newtheorem{lemma}[theorem]{Lemma}
\newtheorem{corollary}[theorem]{Corollary}
\theoremstyle{remark}

\title{Online Agnostic Multiclass Boosting}

\author{%
  Vinod Raman\\
  Department of Statistics\\
  University of Michigan\\
  Ann Arbor, MI 48104 \\
  \texttt{vkraman@umich.edu} \\
   \And
   Ambuj Tewari \\
   Department of Statistics \\
   University of Michigan \\
   Ann Arbor, MI 48018\\
   \texttt{tewaria@umich.edu} \\
}

\begin{document}

\maketitle

\begin{abstract}
Boosting is a fundamental approach in machine learning that enjoys both strong theoretical and practical guarantees. At a high-level, boosting algorithms cleverly aggregate weak learners to generate predictions with arbitrarily high accuracy. In this way, boosting algorithms convert weak learners into strong ones. Recently,  \citet{brukhim2020online}  extended boosting to the online agnostic binary classification setting. A key ingredient in their approach is a clean and simple reduction to online convex optimization, one that efficiently converts an arbitrary online convex optimizer to an agnostic online booster. In this work, we extend this reduction to multiclass problems and give the first boosting algorithm for online agnostic mutliclass classification.  Our reduction also enables the construction of algorithms for statistical agnostic, online realizable, and statistical realizable multiclass boosting. 
\end{abstract}

\section{Introduction}

\textit{Boosting} is a fundamental technique in machine learning that cleverly aggregates the predictions of weak learners to produce a strong learner. Originally studied in the batch (realizable) PAC learning setting for binary classification, boosting has now been extended to a wide variety of settings, including multiclass classification, online learning, and agnostic learning. Recently, \citet{brukhim2020online} extended boosting to online agnostic binary classification, marking the completion of boosting algorithms for all four regimes of statistical/online and agnostic/realizable binary classification. However, less can be said about multiclass classification, where boosting algorithms are only well studied under the assumption of realizability.  Designing  online agnostic multiclass boosting algorithms is important for several reasons. First, realizability is a very strong assumption in real-life: it requires the existence of an expert that \textit{perfectly} labels the data. 
Second, the vast majority of classification tasks require more than two labels, a prominent example being image classification. Lastly, modern machine learning tasks often require sequential processing of data, making the design and development of online algorithms increasingly relevant. In this work we fill this gap in literature by studying online agnostic boosting for multiclass problems. 

\subsection{Main Results}
  We give the first weak learning conditions and algorithms for online agnostic multiclass boosting. A key idea of our algorithms is a reduction from boosting to online convex optimization (OCO), an idea borrowed from \citet{brukhim2020online}. As a consequence of this reduction, we are also able to give algorithms for the three other settings of statistical agnostic,  online realizable, and statistical realizable multiclass boosting. 
  Finally, we give empirical results showcasing that our OCO-based boosting algorithms are fast and competitive with existing state-of-the-art multiclass boosting algorithms. 
\subsection{Related Works}

 Boosting was first studied for binary classification under the realizable PAC learning setting \cite{schapire2013boosting,freund1995boosting,freund1997decision,freund1999adaptive,schapire1990strength} and then later extended to the agnostic PAC learning setting  \cite{feldman2009distribution,gavinsky2003optimally,kanade2009potential,kalai2008agnostic,kalai2005boosting, mansour2002boosting,ben2001agnostic}. The success of boosting for binary classification led to significant interest in designing boosting algorithms for multiclass problems. As a result, several multiclass boosting algorithms were proposed for the realizable batch setting \cite{freund1996experiments,eibl2005multiclass,hastie2009multi,freund1999short}, culminating in the work by \citet{mukherjee2013theory}, who unified the previous approaches under a cost matrix framework.  Beyond the batch setting, online boosting algorithms for binary \cite{oza2001online, chen2012online, beygelzimer2015optimal} and multiclass classification \cite{chen2014boosting, jung2017online} have been designed assuming realizability (mistake-bound). More recently, \citet{brukhim2020online} give the first online \textit{agnostic} (regret-bound) boosting algorithm for binary classification, marking complete all four regimes of statistical/online and agnostic/realizable boosting for {\em binary classification}.
 
Moving to agnostic {\em multiclass} boosting,  \citet{brukhim2021multiclass} study the resources required for boosting in the statistical/batch setting as the number of labels $k$ grows. However, they consider an alternative model of boosting where the weak learner is a strong agnostic PAC learner for a simple "easy-to-learn" base hypothesis class, and the goal is to learn target concepts outside the base class. Specifically, they assume the target concept can be represented by weighted plurality votes over the base class. In this way, the weakness of a weak learner is manifested in the base hypothesis class. Instead, in our work, we consider the standard boosting model of fixing the base hypothesis class, and ask whether a weak learner's  performance can be improved relative to the best fixed hypothesis in that class. 

Beyond classification, several other works have studied online agnostic boosting for real-valued loss functions under both full-information and bandit feedback settings \cite{beygelzimer2015online, agarwal2020boosting, brukhim2021online}. In particular, the work by \citet{brukhim2021online} reduce online boosting for regression tasks under bandit feedback to online linear optimization. A key difference between these works and ours is in the weak learning assumption: these works consider a weak learner that is a strong learner for a small base class of regression functions. The goal of boosting then is to produce a strong online learner for a larger class which contains linear spans of the base class. This is in contrast to this work, where again, we fix the base hypothesis class, and boost the regret bound.


\section{Preliminaries and Notation}
\label{sec:prelim}
We first describe the basic setup for online agnostic multiclass boosting. There are $k$ possible labels $\mathcal{B}_k := \{1,..., k\}$ and $k$ is known to the booster and the weak learners.  The booster maintains $N$ copies of a weak learner, $\mathcal{W}$, which themselves are (randomized) online learning algorithms that sequentially process examples from instance space $\mathcal{X}$ and output predictions in $\mathcal{B}_k$, which we denote as the set of basis vectors of length $k$. At each iteration $t = 1,..., T$, an adversary picks a labeled example $(x_t, y_t) \in \mathcal{X} \times \mathcal{B}_k$ and reveals $x_t$ to the booster. Once the booster observes the unlabeled data $x_t$, it gathers the weak learners’ predictions and makes a final (possibly randomized) prediction $\hat{y_t} \in \mathcal{B}_k$. After observing the booster’s final decision, the adversary reveals the true label $y_t$, and the booster suffers the loss $\mathbbm{1}\{\hat{y}_t \neq y_t\}$ or equivalently, \textit{gains} the reward $2\mathbbm{1}\{\hat{y}_t = y_t\} - 1$. Finally, the booster, after observing the true label $y_t$, updates each weak learner. 
Note that the loss/gain of the booster in round $t$ can be written as $1 - \hat{y}_t \cdot y_t$ and $2\hat{y}_t \cdot y_t - 1$ respectively. We also let $\Delta_k$ represent the $(k-1)$-dimensional probability simplex and $\Delta_{\frac{k}{\gamma}}$ represent the $\frac{1}{\gamma}$-scaled $(k-1)$-dimensional probability simplex for $\gamma \in (0, 1)$. Finally, we let $\mathbbm{1}_k$ denote the $k$-dimensional ones vector.

\textbf{Evaluation}. Unlike the realizable setting, in the agnostic setting, we place no restrictions on how the stream of examples $x_1, ..., x_T$ are labelled. Thus, for a fixed hypothesis class $\mathcal{H} \subseteq \mathcal{B}_k^{\mathcal{X}}$, the goal of the booster is to output predictions $\hat{y_t}$ such that the expected \textit{regret},
$$\E\left[\max_{h \in \mathcal{H}}\sum_{t=1}^T (2h(x_t)\cdot y_t - 1) - \sum_{t=1}^T (2\hat{y}_t\cdot y_t - 1)\right],$$
is minimized, where the expectation is over the randomness of the booster and that of the possibly adaptive adversary. Note, this is in contrast to the realizable setting where the stream is labelled by a $h \in \mathcal{H}$ and we wish the booster to minimize the (expected) number of \textit{mistakes} (mistake-bound). 

\textbf{Agnostic Boosting.} A key technique in agnostic boosting, first appearing in the work by \citet{kanade2009potential}, is to update weak learners by feeding randomly \textit{relabelled} examples. This is in contrast to the realizable setting where we typically update weak learners by passing \textit{reweighted} examples. Accordingly, in order to design a good boosting algorithm, we need to design the booster’s strategy for random relabelling while also quantifying the weak learner’s ability to maximize cumulative gain, even under relabelled data. The first task will be resolved by allowing the booster to use an Online Convex Optimization (OCO) oracle. In this way, we reduce boosting to OCO, an idea borrowed from \citet{brukhim2020online}. For the second task, we give different possible weak learning conditions for the \textit{same} algorithm, all of which capture the ability of a weak learner to maximize cumulative gain with respect to the best fixed competitor in hindsight.

\textbf{Online Convex Optimization.} Our booster will use an OCO oracle to update its weak learners. The OCO setting is a sequential game between an online player and adversary over $N$ rounds (see \cite{hazan2016introduction} for an in-depth introduction). In each round, the player plays a point $x_i$ in a compact convex set $\mathcal{K} \subset \mathbbm{R}^d$, the adversary reveals a loss function $f_i$ chosen from a family of bounded convex functions over $\mathcal{K}$, and the player suffers the loss $f_i(x_i)$. The goal of player is to also minimize \textit{regret}, defined as
$$R(N) = \sum_{i=1}^{N} f_i(x_i) - \min_{x \in \mathcal{K}}\sum_{i=1}^{N}f_i(x).$$
 We will denote an algorithm in this setting as $\text{OCO}(\mathcal{K}, N)$. If $\mathcal{A}$ is a $\text{OCO}(\mathcal{K}, N)$, then we will denote its regret by $R_{\mathcal{A}}(N)$. For many OCO algorithms, like Online Gradient Descent (OGD), the regret $R_{\mathcal{A}}(N)$ is a sub-linear function of the time-horizon $N$.

\section{Online Agnostic Boosting}
\label{sec:oag}
In this section, we present an online agnostic boosting algorithm  and analyze its performance. We begin in Subsection \ref{sec:oag-alg} by formally describing our algorithm which reduces boosting to OCO. Then, in Subsection \ref{sec:oag-reg}, we state a weak learning condition and subsequently prove a regret bound for our proposed algorithm. 

\subsection{Algorithm}
\label{sec:oag-alg}
 Pseudocode for our online agnostic boosting algorithm is provided in Algorithm \ref{OnlineAgnosticBoostFull}. The booster maintains oracle access to $N$ copies of a weak learner $\mathcal{W}_1, ..., \mathcal{W}_N$  as well as a $\text{OCO}(\Delta_k, N)$ algorithm $\mathcal{A}$. Each weak learner is characterized by some advantage parameter $\gamma \in (0, 1]$ (directly proportional to its strength), and we will be precise about what exactly $\gamma$ quantifies in Subsection \ref{sec:oag-reg}. In round $t$, the booster uses the weak learners to make a prediction $\hat{y}_t$, observes the true label $y_t$, and finally simulates a game with the OCO algorithm $\mathcal{A}$ to update each weak learner $\mathcal{W}_i$. Specifically, the booster uses the outputs of $\mathcal{A}$ to feed relabelled examples $(x_t, y_t^i)$ to $\mathcal{W}_i$. A critical component of our algorithm is the $L_2$ projection operator onto the probability simplex $\Delta_k$, which we denote by $\prod$. This is used by the booster in line 4 to make randomized predictions $\hat{y}_t$. 
 
 We now highlight some desirable properties of Algorithm \ref{OnlineAgnosticBoostFull}. First, it is easy to implement and efficient assuming access to an efficient weak learner. In particular, for each round $t$, if the running time of the weak learner is $Q$, then the running time of our booster is $O(NQ + Nk\log(k))$. This is in contrast to the OnlineMBBM algorithm proposed by \citet{jung2017online} for the realizable setting which is not efficient even assuming access to an efficient weak learner. Second, when $k = 2$, our algorithm reduces down to the online agnostic boosting algorithm proposed by \citet{brukhim2020online} in the binary case. Indeed, one can verify that when $k=2$, the $L_2$ projection operator reduces to the same projection used by \citet{brukhim2020online} and the distribution over labels $p_t^i$ induced by the outputs of OCO algorithm are equal for each weak learner in every round $t$. 
\begin{algorithm}
\label{OnlineAgnosticBoostFull}
\KwIn{Weak Learners $\mathcal{W}_1...\mathcal{W}_N$, OCO($\Delta_k$, $N$) algorithm $\mathcal{A}$, Advantage parameter $\gamma$}

\For{$t = 1,...,T$} {

    Receive example $x_t$ \\
    Accumulate weak predictions $h_t = \sum_{i=1}^N\mathcal{W}_i(x_t)$\\
    Set $\mathcal{D}_t = \prod(\frac{h_t}{\gamma N})$\\
    Predict $\hat{y_t} \sim \mathcal{D}_t$ \\
    Receive true label $y_t$ \\
    \For{$i = 1,...,N$} {
        If $i > 1$, obtain $p_t^i = \mathcal{A}(l_t^1,...,l_t^{i-1})$. Else, initialize $p^1_t = \frac{\mathbbm{1}_k}{k}$. \\
        Reveal loss function: $l_t^i(p) = p\cdot \left(\frac{2\mathcal{W}_i(x_t) - \mathbbm{1}_k}{\gamma} - (2y_t - \mathbbm{1}_k)\right)$ \\
        Sample random label $y_t^i \sim p^i_t$, and pass $(x_t, y_t^i)$ to $\mathcal{W}_i$\\
    }
    
    Reset $\mathcal{A}$
}

\caption{Online Agnostic Multiclass Boosting via OCO}
\end{algorithm}

 While the framework in Algorithm \ref{OnlineAgnosticBoostFull} is inspired from \citet{brukhim2020online} in the binary setting, several new complications arise in the multiclass setting.  Most notably, when $k > 2$, we must figure out how the booster should make predictions, what loss functions the booster should construct and pass to $\mathcal{A}$, and lastly, how the booster should use the output of $\mathcal{A}$ to update each weak learner. The interplay behind these three algorithmic pieces is delicate and they have been carefully designed in Algorithm $\ref{OnlineAgnosticBoostFull}$ to enable the analysis in Subsection \ref{sec:oag-reg}. Below, we provide some intuition behind these algorithmic decisions.
 
 \textbf{Randomized Prediction.} At the start of each round, the booster averages the weak learners votes, scales the average by the parameter $\gamma$, projects the scaled vector back into the simplex, and finally samples a random label. When the $L_2$ projection operator is selected, this approach for randomized prediction achieves a \textit{polarization} effect: as $\gamma$ gets smaller, the projection concentrates mass on a fewer number of labels, specifically those labels that have achieved the majority of the votes from the weak learners. When $\gamma$ nears $0$, the $L_2$ projection eventually places all mass on the label with the most votes.

One might think that the $L_2$ projection is not a natural projection operator for $\Delta_k$ (as KL or $L_1$ projections might seem better suited to the geometry of the probability simplex). However, in the case where we project from $\Delta_{\frac{k}{\gamma}}$ to $\Delta_k$, we find that it is a natural choice from a geometric perspective. Figure \ref{fig1} provides a visualization of the $L_2$ projection operator from $\Delta_{\frac{k}{\gamma}}$ to $\Delta_k$ for $k = 2$. Because the spaces $\Delta_{\frac{k}{\gamma}}$ and $\Delta_{k}$ are parallel, given any point $p \in \Delta_{\frac{k}{\gamma}}$, one can think of its $L_2$ projection in the following procedural manner: 
\begin{enumerate}
    \itemsep0em 
    \item Compute the \textit{orthogonal} projection, $\tilde{p}$, onto the plane containing $\Delta_k$.
    \item If $\tilde{p} \in \Delta_k$, output $p^* = \tilde{p}$.
    \item Else, output $p^* = \argmin_{p \in \Delta_k}||p - \tilde{p}||_2$.
\end{enumerate}
\begin{figure}
  \centering
  \includegraphics[width=9cm]{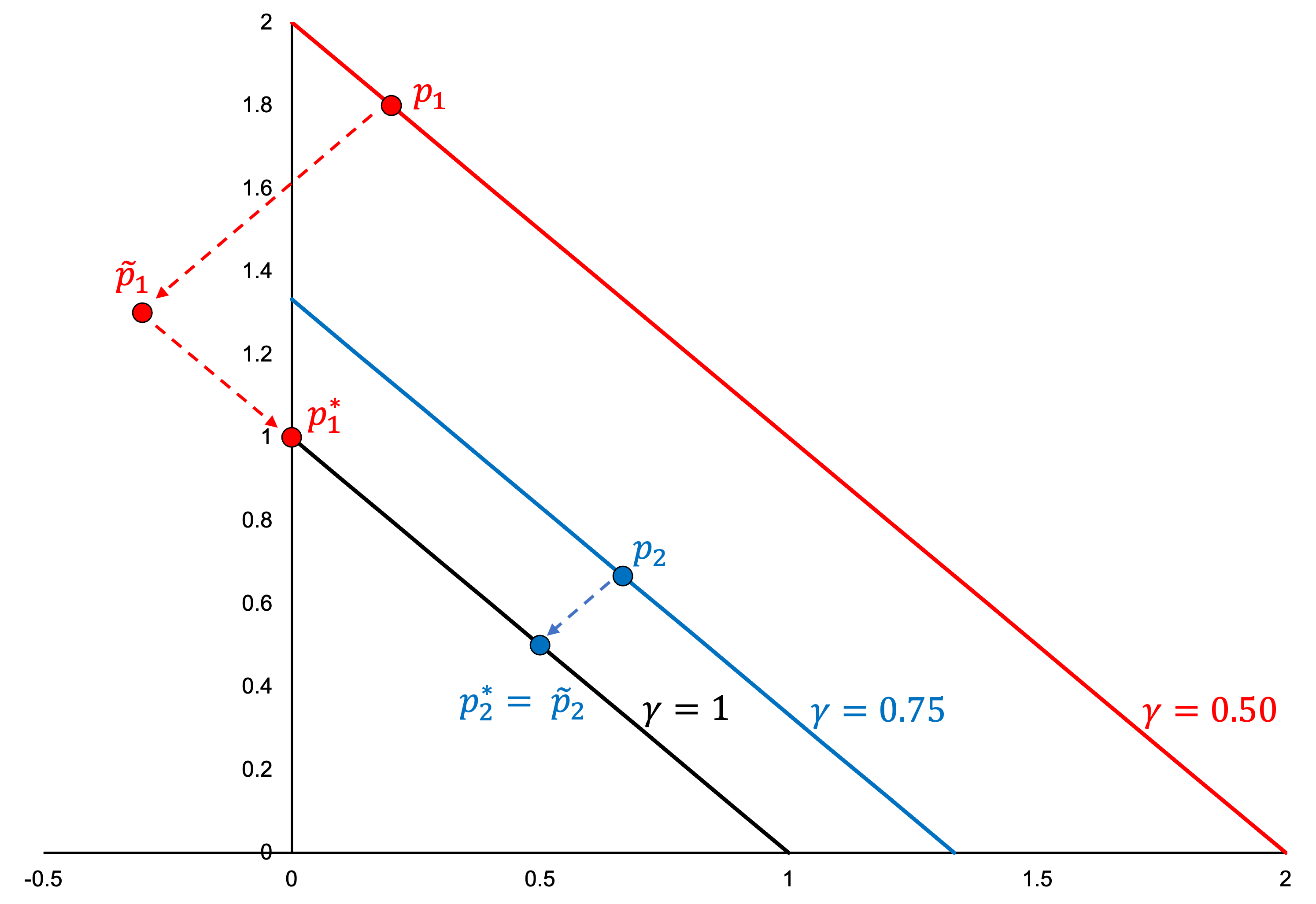}
  \caption{The red, blue, and black lines correspond to $\frac{\Delta_2}{\gamma}$ for $\gamma = 0.50, 0.75, \text{ and, } 1.0$ respectively. $p_1$ and $p_2$ are $\gamma$-scaled votes and $p^*_1$ and $p^*_2$ are their corresponding $L_2$ projection onto $\Delta_2$. $\tilde{p}_1$ and $\tilde{p}_2$ denote the orthogonal projections onto the plane containing $\Delta_2$.}
  \label{fig1}
\end{figure}
From this procedural perspective, Figure \ref{fig1} lends a geometric intuition behind the polarization effect of the $L_2$ projection. As $\gamma$ shrinks, the number of points in  $\Delta_{\frac{k}{\gamma}}$ that lie orthogonally above $\Delta_k$ shrinks. Thus, the orthogonal projection of the vast majority of points in  $\Delta_{\frac{k}{\gamma}}$ will lie outside $\Delta_k$, leading to a greater number of sparse projections that lie on the boundary of $\Delta_k$. From the booster's viewpoint, this property of the $L_2$ projection is desirable. If the weak learners are very weak (small $\gamma$) but somehow concentrate votes on a few labels, then it may be likely that the true response is amongst these few labels. In this sense, $\gamma$ controls how many weak learners need to agree on a particular label to convince the booster to deterministically predict that label. 

\textbf{Updating Weak Learners}. Once the true label $y_t$ is revealed in each round, the booster must update each weak learner. As mentioned in Section \ref{sec:prelim}, one strategy for updating weak learners in the agnostic setting is via \textit{random relabelling}. Indeed, in line 10, the booster passes back to the weak learner the example $x_t$ with a random label $y_t^i \sim p_t^i$. Together, the specified loss function in line 9 and random relabelling strategy in line 10 achieve the following effect: if more weak learners make mistakes, the distribution over labels output by the OCO algorithm $\mathcal{A}$ in line 8 concentrates on the true label $y_t$, increasing the likelihood that $y_t$ is passed to subsequent weak learners. This is desirable as the outputs of the OCO guide each weak learner to correct for the mistakes of preceding learners. 

\subsection{Regret Analysis}
\label{sec:oag-reg}
Before we give the regret bound of Algorithm \ref{OnlineAgnosticBoostFull}, we need to specify the capacity to which the weak learners can make predictions, even under potentially relabelled data. Unfortunately, in the agnostic setting, there is no canonical weak learning condition for multiclass problems. In this paper, we give \textit{several} possible definitions of an Agnostic Weak Online Learner (AWOL) for the multiclass setting, all of which enable a regret analysis of Algorithm \ref{OnlineAgnosticBoostFull}. We emphasize that we can derive different regret bounds for Algorithm \ref{OnlineAgnosticBoostFull} based on what condition we assume our weak learners to satisfy. For this section, we  present a weak learning condition based loosely on a \textit{one-vs-one} approach to multiclass classification. In Appendix \ref{AppendixAltWL}, we provide alternative weak learning conditions.

Define the gain function for input $z \in \mathcal{B}_k$,
 $$\sigma_{y, \ell}(z) = \mathbbm{1}\{z = y\} - \mathbbm{1}\{z = \ell\} = z \cdot (y - \ell).$$
 For some $y \in \mathcal{B}_k$ and $\ell \in \mathcal{B}_k \setminus \{y\}$, $\sigma_{y, \ell}(\cdot)$ can be thought of as the binary classification task between labels $y$ and $\ell$.  Definition \ref{WLC} requires that for any such sequence of binary classification tasks, an online agnostic weak learner must be able to eventually distinguish between every pair of labels to some non-trivial, but far from optimal, degree.

 \begin{definition}[Agnostic Weak Online Learning]
 \label{WLC}
 Let $\mathcal{H} \subseteq \mathcal{B}_k^{\mathcal{X}}$ be a class of experts and let $0 < \gamma \leq 1$ denote the \say{advantage}.  An online learning algorithm $\mathcal{W}$ is a $(\gamma, T)$-agnostic weak online learner (\textbf{AWOL}) for $\mathcal{H}$ if for any adaptively chosen sequence of tuples $(x_t, y_t, \ell_t) \in \mathcal{X} \times \mathcal{B}_k \times\mathcal{B}_k $ where $\ell_t \neq y_t$,  the algorithm outputs $\mathcal{W}(x_t) \in \mathcal{B}_k$ at every iteration $t \in [T]$ such that,
  \[\gamma \max_{h \in \mathcal{H}} \E\left[ \sum_{t = 1}^T \sigma_{y_t, \ell_t}(h(x_t))\right]   -  \E\left[\sum_{t=1}^T \sigma_{y_t, \ell_t}(\mathcal{W}(x_t))\right] \leq R_\mathcal{W}(T,k),\]
   where the expectation is taken w.r.t. the randomness of the weak learner $\mathcal{W}$ and that of the possibly adaptive adversary, and  $R_\mathcal{W} : \mathbbm{N} \times \mathbbm{N}  \rightarrow \mathbbm{R}_+$ is the additive regret: a non-decreasing, sub-linear function of $T$.
  \end{definition}
  
   The precise dependence of $R_{\mathcal{W}}(T,k)$ on $k$ is explored in more detail in Appendix \ref{deponk} where we explicitly construct learners satisfying AWOL. We make few remarks about Definition \ref{WLC}. First, the strength of the weak learner varies directly with $\gamma$. Second, for $k = 2$, Definition \ref{WLC} reduces to the weak learning condition by \citet{brukhim2020online} in the binary setting. Finally, we emphasize that Definition \ref{WLC} holds under an \textit{adaptive} adversary, one that can choose $(y_t, \ell_t)$ based on $\{\mathcal{W}(x_i)\}_{i=1}^{t-1}$ and its own internal random bits. Importantly, we also allow the adversary to pick $\ell_t$ even after it has observed $\mathcal{W}(x_t)$ as this only controls how much loss the weak learner suffers when it is incorrect. 
 
  Under the assumption that weak learners satisfy Definition \ref{WLC}, Theorem \ref{thm1} bounds the regret of Algorithm \ref{OnlineAgnosticBoostFull} under an \textit{oblivious} adversary. Using a standard reduction, our results can then be generalized to an \textit{adaptive} adversary (see Chapter 4 in \cite{10.5555/1137817}). In particular, a key requirement allowing an oblivious regret bound to generalize to an adaptive regret bound is that the learner’s predictions on round $t$ should not depend on any of its past predictions from previous rounds. This is indeed true for our Booster. 
  
    \begin{theorem}[Regret Bound]
    \label{thm1}
    Assuming weak learners satisfy Definition \ref{WLC}, the expected regret bound of Algorithm \ref{OnlineAgnosticBoostFull} is,
    \[\frac{1}{T}\E\left[\max_{h \in \mathcal{H}}\sum_{t=1}^{T} \left(2h(x_t) \cdot y_t - 1\right)  - \sum_{t=1}^{T}(2\hat{y}_t \cdot y_t - 1)\right] \leq \frac{R_{\mathcal{W}}(T, k)}{\gamma T} + \frac{R_{\mathcal{A}}(N)}{N},\]
    where the expectation is over the randomness of the algorithm and weak learners, and $R_{\mathcal{W}}(T, k)$, $R_{\mathcal{A}}(N)$ are the regret terms of the AWOL, OCO algorithms respectively. 
    \end{theorem}
 
If one picks  $\mathcal{A}$ to be \textit{Online Gradient Descent} (OGD), then $R_{\mathcal{A}}(N) = O(GD \sqrt{N})$, where $D$ is the diameter of $\Delta_k$ and $G$ is the upper-bound on $||\nabla_pl_t^i(p)||$. In our setting, $D = \sqrt{2}$ and $G = O(\frac{1}{\gamma})$ (using Lemma \ref{equivalence}) and  hence the average regret further simplifies to: 
\[\frac{1}{T}\E\left[\max_{h \in \mathcal{H}}\sum_{t=1}^{T} \left(2h(x_t) \cdot y_t - 1\right)  - \sum_{t=1}^{T}(2\hat{y}_t \cdot y_t - 1)\right] \leq \frac{R_{\mathcal{W}}(T, k)}{\gamma T} + O\left(\frac{1}{\gamma\sqrt{N}}\right).\]
To exemplify the role of $N$, consider a scenario where  $R_{\mathcal{W}}(T, k) = O(k\sqrt{T})$ (see Appendix \ref{deponk} for examples). By setting $N = \frac{T}{\gamma^2}$  the overall regret of the booster becomes $O(\frac{k\sqrt{T}}{\gamma})$. In the next subsection, we give the proof of Theorem \ref{thm1}. 
  
\subsubsection{Proof of Theorem \ref{thm1}}
\label{proofthm1}
 We follow a similar procedure to \citet{brukhim2020online} by lower and upper bounding the expected sum of losses passed to $\mathcal{A}$ in terms of the regret of the weak learner and the regret of $\mathcal{A}$ respectively. These bounds rely on several lemmas that have been abstracted out and provided in Appendix \ref{AppendixLemmas}.  As mentioned previously, we also assume an oblivious adversary.

Let $(x_1, y_1), ..., (x_T, y_T)$ be any sequence of example-label pairs. We start by giving a lower bound on the expected sum of losses passed to $\mathcal{A}$ using Definition $\ref{WLC}$. Define  $h^* = \argmax_{h \in \mathcal{H}}\sum_{t=1}^T  (2h(x_t)\cdot y_t - 1)$ as the optimal competitor in hindsight and let 

$$\ell_t^i = \begin{cases}
      \mathcal{W}_i(x_t), & \text{if}\ \mathcal{W}_i(x_t) \neq y_t^i \\
      \ell \in \mathcal{B}_k \setminus \{y_t^i\}, & \text{otherwise}
    \end{cases}.$$
Note that the precise choice of $\ell_t^i$ in the second case is not important. The proof below only uses the fact that $\ell_t^i \in \mathcal{B}_k \setminus \{y_t^i\}$ in that case. We then have, 
    \begin{align*}
        \E\left[\sum_{i=1}^{N} \sum_{t=1}^{T} l_t^i(p_t^i)\right] &= \E\left[\sum_{i=1}^{N} \sum_{t=1}^{T} p_t^i \cdot \left(\frac{2\mathcal{W}_i(x_t) - \mathbbm{1}_k}{\gamma} - (2y_t - \mathbbm{1}_k) \right)\right] \\
        &= \frac{1}{\gamma}\sum_{i=1}^{N} \sum_{t=1}^{T}\E\left[p_t^i \cdot (2\mathcal{W}_i(x_t) - \mathbbm{1}_k)\right]   - \sum_{i=1}^{N} \sum_{t=1}^{T} \E\left[p_t^i \cdot (2y_t - \mathbbm{1}_k)\right] \\
        &= \frac{1}{\gamma}\sum_{i=1}^{N} \sum_{t=1}^{T}\E\left[ 2\mathcal{W}_i(x_t) \cdot y_t^i - 1\right]   - \sum_{i=1}^{N} \sum_{t=1}^{T} \E\left[p_t^i \cdot (2y_t - \mathbbm{1}_k)\right] && \text{(Lemma \ref{expectation})}\\
        &= \frac{1}{\gamma}\sum_{i=1}^{N} \sum_{t=1}^{T}\E\left[ \sigma_{y_t^i, \ell_t^i}(\mathcal{W}_i(x_t))\right]   - \sum_{i=1}^{N} \sum_{t=1}^{T} \E\left[p_t^i \cdot (2y_t - \mathbbm{1}_k)\right].
    \end{align*}
    
Using the weak learning condition in Definition \ref{WLC},

    \begin{align*}
        \frac{1}{\gamma}\sum_{i=1}^{N} \sum_{t=1}^{T}\E\left[ \sigma_{y_t^i, \ell_t^i}(\mathcal{W}_i(x_t))\right] &\geq \sum_{i=1}^{N} \max_{h \in \mathcal{H}}\sum_{t=1}^{T}\E\left[ \sigma_{y_t^i, \ell_t^i}(h(x_t))\right] - \frac{NR_{\mathcal{W}}(T, k)}{\gamma} && \text{(Definition \ref{WLC})}\\
        &\geq \sum_{i=1}^{N} \sum_{t=1}^{T}\E\left[ \sigma_{y_t^i, \ell_t^i}(h^*(x_t))\right] -  \frac{NR_{\mathcal{W}}(T, k)}{\gamma}\\
        &\geq \sum_{i=1}^{N} \sum_{t = 1}^T\E\left[ 2h^*(x_t) \cdot y_t^i - 1)\right]  - \frac{NR_{\mathcal{W}}(T, k)}{\gamma}\\
        &= \sum_{i=1}^{N}\sum_{t=1}^{T}\E\left[ p_t^i \cdot (2h^*(x_t) - \mathbbm{1}_k)\right] - \frac{NR_{\mathcal{W}}(T, k)}{\gamma}. && \text{(Lemma \ref{expectation})}
    \end{align*}
    
Putting things together, we find, 
    \begin{align*}
        \E\left[\sum_{i=1}^{N} \sum_{t=1}^{T} l_t^i(p_t^i)\right] &\geq \sum_{i=1}^{N}\sum_{t=1}^{T}\E\left[ p_t^i \cdot (2h^*(x_t) - 2y_t) \right] - \frac{NR_{\mathcal{W}}(T, k)}{\gamma}\\
        &\geq \sum_{i=1}^{N}\sum_{t=1}^{T}\E\left[ 2(h^*(x_t) \cdot y_t - 1) \right] - \frac{NR_{\mathcal{W}}(T, k)}{\gamma} && \text{(Lemma \ref{lowerbound})}\\
        &= N\sum_{t=1}^{T} 2(h^*(x_t) \cdot y_t - 1)  - \frac{NR_{\mathcal{W}}(T, k)}{\gamma}.\\
    \end{align*}

%
 Now, we compute an upper bound. For any $t \in \left[T \right]$ and arbitrary $p^* \in \Delta_k$:
    \begin{align*}
        \E\left[\frac{1}{N}\sum_{i=1}^{N} l_t^i(p_t^i)\right] &\leq \frac{1}{N}\E\left[\min_{p \in \mathcal{K}}\sum_{i=1}^{N}l^i_t(p)\right] + \frac{R_{\mathcal{A}}(N)}{N} && \text{(OCO Regret)}\\
        &\leq \frac{1}{N}\E\left[\sum_{i=1}^{N}l^i_t(p^*)\right] + \frac{R_{\mathcal{A}}(N)}{N} \\
        &= \frac{1}{N}\sum_{i=1}^{N}\E\left[p^* \cdot \left(\frac{2\mathcal{W}_i(x_t) - \mathbbm{1}_k}{\gamma} - (2y_t - \mathbbm{1}_k)\right)\right] + \frac{R_{\mathcal{A}}(N)}{N} \\
        &= \frac{1}{\gamma N}\sum_{i=1}^{N} \left(p^* \cdot \E\left[2\mathcal{W}_i(x_t) - \mathbbm{1}_k \right]\right) - p^* \cdot (2y_t - \mathbbm{1}_k) + \frac{R_{\mathcal{A}}(N)}{N} \\
        &= p^* \cdot \left(\frac{1}{\gamma N}\sum_{i=1}^{N} \E\left[2\mathcal{W}_i(x_t) - \mathbbm{1}_k \right] \right) - p^* \cdot (2y_t - \mathbbm{1}_k) + \frac{R_{\mathcal{A}}(N)}{N} \\
        &= \E\left[p^* \cdot \left( \frac{\frac{2}{N}\sum_{i=1}^N\mathcal{W}_i(x_t) - \mathbbm{1}_k}{\gamma} - (2y_t - \mathbbm{1}_k) \right) \right] + \frac{R_{\mathcal{A}}(N)}{N} \\
        &\leq \E\left[2\left(\prod\left(\frac{1}{\gamma N}\sum_{i=1}^N \mathcal{W}_i(x_t)\right) \cdot y_t - 1\right)\right]  + \frac{R_{\mathcal{A}}(N)}{N} && \text{(Lemma \ref{upperbound})}\\
        &= 2\left(\E\left[\hat{y}_t\right] \cdot y_t - 1 \right) + \frac{R_{\mathcal{A}}(N)}{N}. && \text{(Law of total expectation)}\\
    \end{align*}
Summing over $T$,
\[\E\left[\frac{1}{N}\sum_{t=1}^{T}\sum_{i=1}^{N} l_t^i(p_t^i)\right] \leq \sum_{t=1}^{T}2 \left( \E\left[\hat{y}_t \right] \cdot y_t - 1 \right) + \frac{TR_{\mathcal{A}}(N)}{N}.\]

Combining lower and upper bounds for $\E\left[\frac{1}{NT}\sum_{i=1}^{N} \ell_t^i(p_t^i)\right]$, we get
\[\frac{1}{T}\E\left[\max_{h \in \mathcal{H}}\sum_{t=1}^{T} \left(2h(x_t) \cdot y_t - 1\right)  - \sum_{t=1}^{T}(2\hat{y}_t \cdot y_t - 1)\right] \leq \frac{R_{\mathcal{W}}(T, k)}{\gamma T} + \frac{R_{\mathcal{A}}(N)}{N},\]
which completes the proof. Since $2h(x_t) \cdot y_t - 1 = 1 - 2\mathbbm{1}_k\{h(x_t) \neq y_t\}$, we can also write the average regret in terms of the number of \textit{mistakes},
\[\frac{1}{T}\E\left[\sum_{t=1}^{T}\mathbbm{1}\{\hat{y}_t \neq y_t\} - \min_{h \in \mathcal{H}}\sum_{t=1}^{T} \mathbbm{1}\{h(x_t) \neq y_t\} \right]  \leq \frac{R_{\mathcal{W}}(T, k)}{2\gamma T} + \frac{R_{\mathcal{A}}(N)}{2N}.\]
%


\section{Beyond Online Agnostic Boosting}
\label{sec:beyond}
In this section, we present results of extending our reduction to the three other boosting settings, namely statistical agnostic, online realizable, and statistical realizable learning. The purpose of this section is to showcase the generality of the OCO-based boosting framework and not to achieve state-of-the-art bounds for these settings. Boosting algorithms and all associated proofs can be found in Appendix\ref{AppendixStatAgnBoost} and  \ref{AppendixRealBoost}. Throughout this section we will let $\mathcal{W}_{S}$ denote the hypothesis output by a weak learner trained on a sample $S$.

\subsection{Statistical Agnostic Boosting}
In the statistical setting, our objective of interest is the \textit{correlation}, which we define below. Let $\mathcal{D}$ be a distribution over $\mathcal{X} \times \mathcal{B}_k$ and let $h: \mathcal{X} \rightarrow \mathcal{B}_k$ be an hypothesis. Define the \textit{multiclass} correlation of $h$ with respect to $\mathcal{D}$ as
$$cor_D(h) = \E_{(x, y) \sim D} \left[h(x)\cdot (2y - \mathbbm{1}_k)\right].$$
Like the online agnostic setting, the boosting algorithm for this setting (provided in Appendix  \ref{AppendixStatAgnBoost}) can be analyzed under several candidate weak learning conditions. To showcase the dependence on $k$, we provide a weak learning condition based loosely on a \textit{one-vs-all} approach to multiclass classification. 

\begin{definition}[Empirical Agnostic Weak Learning]
\label{WLCS}
 Let $\mathcal{H} \subseteq \mathcal{B}_k^{\mathcal{X}}$ be a hypothesis class and let $0 < \gamma \leq 1$ denote the \say{advantage}. Let $\textbf{x} = (x_1, ..., x_m) \in \mathcal{X}$ denote an unlabeled sample. A learning algorithm $\mathcal{W}$ is a $(\gamma, \epsilon_0, m_0)$-agnostic weak learner (\textbf{AWL}) for $\mathcal{H}$ with respect to \textbf{x} if for any labels $\textbf{y} = (y_1, ..., y_m) \in \mathcal{B}_k$, and every reference label $\ell \in \mathcal{B}_k$, 
 $$ \E_{S'} \left[ \sum_{i:y_i = \ell} \mathcal{W}_{S'}(x_i) \cdot (2y_i - \mathbbm{1}_k) \right] \geq \gamma \max_{h \in \mathcal{H}} \sum_{i:y_i = \ell} h(x_i) \cdot (2y_i - \mathbbm{1}_k) - m\epsilon_0,$$
 where $S'$ is an independent sample of size $m_0$ drawn from the distribution which uniformly assigns to each example $(x_i, y_i)$ probability $1/m$.
\end{definition}

 Under the assumption that our weak learners satisfy Definition \ref{WLCS}, Corollary \ref{corAgnStat} bounds the expected correlation of a statistical agnostic boosting algorithm.

\begin{corollary}[Empirical Agnostic Correlation Bound]
\label{corAgnStat}
There exists a boosting algorithm whose output after $T$ rounds, denoted by $\bar{h}$, satisfies: 
$$\E\left[cor_S(\bar{h})\right] \geq \max_{h \in \mathcal{H}} \E\left[cor_S(h)\right] - \frac{R_{\mathcal{A}}(T)}{T} - \frac{k\epsilon_0}{\gamma},$$
where $S$ is the distribution which uniformly assigns to each example $(x_i, y_i)$ probability $1/m$.
\end{corollary}

Letting $\mathcal{A}$ be OGD, and setting $T = O(\frac{1}{\gamma^2 \epsilon^2})$ for any $\epsilon > 0$ gives an error rate of $\frac{k\epsilon_0}{\gamma} + \epsilon$. Although, to our best knowledge, there are no existing multiclass boosting algorithms in the agnostic setting, we point out that there is evidence to suggest that a sub-optimal dependence of the error rate on $k$ might be unavoidable (see \cite{brukhim2021multiclass}).
\subsection{Online Realizable Boosting}
In the online realizable setting, we are guaranteed that the stream of examples $(x_1, y_1),...,(x_T, y_T)$ is perfectly labelled by some hypothesis $h \in \mathcal{H}$. Definition \ref{WLCOR} then gives an appropriate online realizable weak learning condition.

\begin{definition}[Realizable Weak Online Learning]
\label{WLCOR}
Let $\mathcal{H} \subseteq \mathcal{B}_k^{\mathcal{X}}$ be a class of experts and let $0 < \gamma \leq 1$ denote the \say{advantage}. An online learning algorithm $\mathcal{W}$ is a $(\gamma, T)$-realizable weak online learner (\textbf{RWOL}) for $\mathcal{H}$ if for any sequence $(x_1, y_1), ..., (x_T, y_T) \in \mathcal{X} \times \mathcal{B}_k$ that is realizable by $\mathcal{H}$, at every iteration $t \in [T]$, the algorithm outputs $\mathcal{W}(x_t) \in \mathcal{B}_k$ such that, 
$$\E\left[ \sum_{t=1}^{T} (2\mathcal{W}(x_t) \cdot y_t - 1)\right] \geq \gamma T - R_\mathcal{W}(T, k),$$
 where the expectation is taken w.r.t. the randomness of the weak learner $\mathcal{W}$ and $R_\mathcal{W} : \mathbbm{N} \times \mathbbm{N} \rightarrow \mathbbm{R}_+$ is the additive regret: a non-decreasing, sub-linear function of $T$.
\end{definition}

Under Definition $\ref{WLCOR}$,  Corollary \ref{corOnlReal} bounds the expected gain of an online boosting algorithm.

\begin{corollary}[Mistake Bound]
\label{corOnlReal}
There exists a boosting algorithm whose outputs $\hat{y}_1, ..., \hat{y}_T$ satisfy:
$$\frac{1}{T}\sum_{t=1}^T (2\E\left[\hat{y}_t\right] \cdot y_t - 1)  \geq 1 - \frac{R_{\mathcal{A}}(N)}{N} -\frac{\tilde{R}_{\mathcal{W}}(T, k)}{\gamma T},$$
where $\tilde{R}_{\mathcal{W}}(T, k) = 2R_{\mathcal{W}}(T, k) + \tilde{O}(\sqrt{T})$.
\end{corollary}

If we consider a scenario where $\tilde{R}_{\mathcal{W}}(T, k) = O(k\sqrt{T})$, taking $\mathcal{A}$ to be OGD, $N$ to be $O(\frac{1}{\gamma^2 \epsilon^2})$ and $T$ to be $O(\frac{1}{\gamma^2 \epsilon^2})$, gives a mistake-bound at most $\epsilon k T$. We point out that we cannot readily compare this bound with the existing bounds for online realizable boosting by \citet{jung2017online} because the dependence of $R_{\mathcal{W}(T, k)}$ on $k$ depends on the weak learner of choice. 

\subsection{Statistical Realizable Boosting}
In the statistical realizable setting, our metric of interest is again the multiclass correlation. However, under the realizability assumption,  $\max_{h^* \in \mathcal{H}} cor_D(h^*) = 1$. Accordingly, Definition \ref{WLCSR} gives an appropriate realizable weak learning condition.

\begin{definition}[Empirical Realizable Weak Learning]
 \label{WLCSR}
 Let $\mathcal{H} \subseteq \mathcal{B}_k^{\mathcal{X}}$ be a hypothesis class and let $0 < \gamma \leq 1$ denote the \say{advantage}.  Let $S = \{(x_1, y_1), ..., (x_m, y_m) \} \in \mathcal{X} \times \mathcal{B}_k$ be a sample. A learning algorithm $\mathcal{W}$ is a $(\gamma, m_0)$ - realizable weak learner (\textbf{RWL}) for $\mathcal{H}$ with respect to $S$ if for any distribution $\textbf{p} = (p_1, ..., p_m)$ over the examples, 
  $$ \E_{S'} \left[cor_{\textbf{p}}(\mathcal{W}_{S'}) \right] \geq \gamma ,$$
where $S'$ is an independent sample of size $m_0$ drawn from $\textbf{p}$.
\end{definition}

Under Definition \ref{WLCSR}, Corollary \ref{corRealStat} bounds the expected correlation of a statistical realizable boosting algorithm. 

\begin{corollary}[Empirical Correlation Bound]
\label{corRealStat}
There exists a boosting algorithm whose output after $T$ rounds, denoted $\bar{h}$, satisfies,
$$\E\left[cor_S(\bar{h})\right] \geq 1 - \frac{R_{\mathcal{A}}(T)}{T},$$
where $S$ is the distribution which uniformly assigns to each example $(x_i, y_i)$ probability $1/m$.
\end{corollary}

Note that the cost of weak learning manifests in the regret term of the OCO algorithm. Precisely, if one picks OGD to be the OCO algorithm, then the bound on correlation can be expressed as 
$\E\left[cor_S(\bar{h}) \right] \geq 1 - O\left(\frac{1}{\gamma \sqrt{T}}\right)$,
which exhibits a decreasing lower bound as $\gamma$ shrinks. Setting $T = O(\frac{1}{\gamma^2 \epsilon^2})$ for any $\epsilon > 0$ ensures that at most $\epsilon$ error is obtained. It is difficult to compare our error rates to that achieved by existing multiclass boosting algorithm in the realizable setting  because our weak learning condition is different and potentially stronger. Nevertheless, our bounds are  sub-optimal to those in \citet{mukherjee2013theory} which only require setting $T = O(\frac{1}{\gamma^2}\log(\frac{k}{\epsilon}))$ to achieve $\epsilon$ error. 

\section{Experiments}
\label{sec:exp}
We performed experiments with Algorithm \ref{OnlineAgnosticBoostFull} on seven UCI datasets \cite{DuBois:2008}. Since, to our knowledge,  there are no known online {\em agnostic} boosting algorithms for multiclass problems, we benchmark performance against three online \textit{realizable} multiclass boosting algorithms: the state-of-the-art optimal (OnlineMBBM) and adaptive (Adaboost.OLM) online boosting algorithms by \citet{jung2017online}, and the OCO-based online boosting algorithm alluded to in Section \ref{sec:beyond} (see Appendix \ref{AppendixRealBoost}). For weak learners, we used the implementation of the VeryFastDecisionTree from the River package \cite{2020river} and restricted the maximum depth of the tree to 1. We used Projected OGD \cite{zinkevich2003online} for the OCO algorithm and set the number of weak learners, $N$, to $100$ for each boosting algorithm.  

Our experiments using Algorithm \ref{OnlineAgnosticBoostFull} are performed using \textit{fractional relabeling}, a technique borrowed from \citet{kanade2009potential}. That is, instead of passing just a single example-label pair $(x_t, y_t^i)$ where $y_t^i \sim p_t^i$, we pass all $k$ possible \textit{weighted} example-label tuples $\{(x_t, \ell, p_t^i[\ell])\}_{\ell=1}^{k}$  to weak learner $i$ in round $t$. Experiments with random relabeling showed that random relabeling runs faster but performs worse than fractional relabeling.

Table \ref{tab:data} summarizes the average accuracy and runtime over five independent shuffles of each dataset for each boosting algorithm.  "Agn" refers to Algorithm \ref{OnlineAgnosticBoostFull}, "Opt" to OnlineMBBM, "Ada" to Adaboost.OLM, and "OCOR" to the OCO-based online realizable algorithm.  Standard errors for accuracies are provided in Appendix \ref{ExpDet}. All algorithms \textbf{except} AdaBoost.OLM are parameterized by an advantage parameter $\gamma \in (0, 1)$.  Thus, $\gamma$ was tuned separately for each respective cell of Table \ref{tab:data}.  See Appendix \ref{ExpDet} for more experimental details. All code is available at \url{https://github.com/vinodkraman/OnlineAgnosticMulticlassBoosting}.

%
\begin{table}
\caption{Average accuracy and runtime of algorithms on 7 UCI datasets.}
\centering
\begin{tabular}{| c | c | c | c | c| c| c | c| c|c|c|}
\hline
\textbf{Dataset} & \textbf{$k$} & \multicolumn{4}{ c |}{\textbf{Accuracy(\%)}}&
\multicolumn{4}{ c |}{\textbf{Runtime(s)}}\\ 
\cline{3-10}
 && Agn & Opt & Ada & OCOR & Agn & Opt & Ada & OCOR \\
\hline 
Balance  & 3 & $\textbf{84.4}$ & $79.5$ & $75.2$ & $78.8$ & 11.5 & 57.6 & 7.1 &  4.5\\ 
Cars  & 4 & $\textbf{80.6}$ & $69.9$ & $75.8$ & $69.6$ & 42.5 & 104.1 & 20.4 & 12.1\\ 
Landsat  & 6 & $67.0$  & $\textbf{80.8}$ & $56.9$ & $79.6$ & 1255.6 & 1814.3 & 81.8 & 440.4\\ 
Segment & 7 & $75.0$ & $78.9$ & $68.6$ & $\textbf{79.3}$ & 647.2 & 2448.3 & 51.8 &  154.7\\ 
Mice & 8 & $\textbf{86.0}$ & $77.6$ & $71.3$ & $79.6$ & 1025.8 & 1811.6 & 90.1 &  258.2\\
Yeast & 10 & $42.3$ & $39.8$ & $41.7$ & $\textbf{47.6}$ & 348.3 & 1468.8 & 25.2 & 56.6\\
Abalone & 28 & $22.1$ & $\textbf{24.9}$ & $19.2$ & $22.0$ & 4036.7 & 9775.3 & 108.0 & 248.2\\ 
\hline
\end{tabular}
\label{tab:data}
\end{table}

Despite having the fastest overall runtime, Adaboost.OLM had the poorest performance. Compared to OnlineMBBM, our OCO-based boosting algorithms achieve comparable performance at a fraction of the runtime. Specifically, for three of our datasets, our online agnostic boosting algorithm achieves the highest accuracy. For the remaining datasets, our OCO-based realizable boosting algorithm achieves comparable accuracy to OnlineMBBM with shorter runtimes. 

\section{Discussion}
\label{sec:discussion}
 We give the first weak learning conditions and algorithm for online agnostic multiclass boosting. Our algorithm relies on a clean and simple reduction from boosting to online convex optimization. This fruitful connection allows us to go beyond the online agnostic setting and design multiclass boosting algorithms for all four regimes of statistical/online and agnostic/realizable learning.  As future work, we leave it open to identify the correct weak learning condition, construct adaptive versions of our boosting algorithms, improve regret upper bounds/prove lower bounds, design agnostic boosting algorithms under bandit feedback, and study the impact of the choice of weak learner and OCO algorithm on the empirical performance of our OCO-based boosting algorithms. 

\paragraph{Acknowledgements.} AT acknowledges the support of NSF via grant IIS-2007055.

\newpage
\bibliography{neurips_2022}

\section*{Checklist}

\begin{enumerate}

\item For all authors...
\begin{enumerate}
  \item Do the main claims made in the abstract and introduction accurately reflect the paper's contributions and scope?
    \answerYes{}
  \item Did you describe the limitations of your work?
    \answerYes{See Section \ref{sec:discussion} where we identify future work.}
  \item Did you discuss any potential negative societal impacts of your work?
    \answerNo{Due to the theoretical/general nature of the work, we do not foresee any major/direct ethical or societal consequences of the research we presented.}
  \item Have you read the ethics review guidelines and ensured that your paper conforms to them?
    \answerYes{}
\end{enumerate}

\item If you are including theoretical results...
\begin{enumerate}
  \item Did you state the full set of assumptions of all theoretical results?
    \answerYes{See Section \ref{sec:oag-reg}, Section \ref{sec:beyond}, Appendix \ref{AppendixStatAgnBoost},  
    Appendix \ref{AppendixRealBoost}}, and Appendix \ref{AppendixAltWL}
        \item Did you include complete proofs of all theoretical results?
    \answerYes{See Section \ref{proofthm1}, Appendix \ref{AppendixStatAgnBoost}, Appendix \ref{AppendixRealBoost}, Appendix \ref{deponk},
    and Appendix \ref{AppendixLemmas} }
\end{enumerate}

\item If you ran experiments...
\begin{enumerate}
  \item Did you include the code, data, and instructions needed to reproduce the main experimental results (either in the supplemental material or as a URL)?
    \answerYes{In order to preserve anonymity, we will include a URL to a github repository if accepted}
  \item Did you specify all the training details (e.g., data splits, hyperparameters, how they were chosen)?
    \answerYes{See Section \ref{sec:exp} and Appendix \ref{ExpDet}}
        \item Did you report error bars (e.g., with respect to the random seed after running experiments multiple times)?
    \answerYes{See Appendix \ref{ExpDet}}
        \item Did you include the total amount of compute and the type of resources used (e.g., type of GPUs, internal cluster, or cloud provider)?
    \answerYes{See Appendix \ref{ExpDet}}
\end{enumerate}

\item If you are using existing assets (e.g., code, data, models) or curating/releasing new assets...
\begin{enumerate}
  \item If your work uses existing assets, did you cite the creators?
    \answerYes{See Section \ref{sec:exp}}
  \item Did you mention the license of the assets?
    \answerYes{See Appendix \ref{ExpDet}}
  \item Did you include any new assets either in the supplemental material or as a URL?
    \answerNA{}
  \item Did you discuss whether and how consent was obtained from people whose data you're using/curating?
    \answerNA{}
  \item Did you discuss whether the data you are using/curating contains personally identifiable information or offensive content?
    \answerNA{None of the datasets we are using contain information about people.}
\end{enumerate}

\item If you used crowdsourcing or conducted research with human subjects...
\begin{enumerate}
  \item Did you include the full text of instructions given to participants and screenshots, if applicable?
    \answerNA{}
  \item Did you describe any potential participant risks, with links to Institutional Review Board (IRB) approvals, if applicable?
    \answerNA{}
  \item Did you include the estimated hourly wage paid to participants and the total amount spent on participant compensation?
    \answerNA{}
\end{enumerate}

\end{enumerate}

\newpage
\appendix

\section{Statistical Agnostic Boosting}
\label{AppendixStatAgnBoost}

 We fill in the gaps of Section \ref{sec:beyond} for the statistical agnostic setting by providing a boosting algorithm and its theoretical analysis. 
\subsection{Algorithm}
 We now describe a statistical agnostic boosting algorithm whose pseudocode is provided below. The booster is given as input a sample $S = (x_1, y_1), ..., (x_m, y_m) \in \mathcal{X} \times \mathcal{B}_k$. The booster has black-box oracle access to $T$ copies of a $(\gamma, \epsilon_0, m_0)$-AWL algorithm,
 $\mathcal{W}_1, ..., \mathcal{W}_T$, each satisfying Definition \ref{WLCS}, and $m$ copies of an $\text{OCO}(\Delta_k, T)$ algorithm, $\mathcal{A}^1, ..., \mathcal{A}^m$.  Importantly, note that in line 8 , we denote by $\prod$ as a randomized \textit{prediction} operator that given a $\gamma$-scaled distribution over $k$ labels, first computes the $L_2$ projection onto $\Delta_k$, and then randomly samples a label from the resulting distribution.
 
\begin{algorithm}
\label{StatAgnosticBoost}
\KwIn{$(\gamma, \epsilon_0, m_0)$-AWLs $\mathcal{W}_1...\mathcal{W}_T$, OCO($\Delta_k$, $T$) algorithms $\mathcal{A}^1...\mathcal{A}^m$, and Sample $S = (x_1, y_1)...(x_m, y_m) \in \mathcal{X} \times \mathcal{B}_k$}

\For{$t = 1,...,T$} {
    If $t = 1$, set $P_1[i] = \frac{\mathbbm{1}_k}{k}$ $\forall i\in [m]$\\
    Pass $m_0$ examples to $\mathcal{W}_t$ drawn from the following distribution: \\
    Draw $x_i$ w.p. $\frac{1}{m}$ and relabel by drawing $\tilde{y}_i \sim P_t[i]$ \\
    Let $h_t$ be the weak hypothesis returned by $\mathcal{W}_t$ \\
    Reveal loss function: $l^i_t(p) = (2p - \mathbbm{1}_k)\cdot \left(\frac{h_t(x_i)}{\gamma} - y_i \right)$ $\forall i\in[m]$ \\
    Set $P_{t+1}[i] =\mathcal{A}^i(l_1^i, ..., l_t^i)$ $\forall i \in [m]$
}

return $\bar{h}(x) = \prod\left(\frac{1}{\gamma T} \sum_{t = 1}^{T} h_t(x)\right)$

\caption{Statistical Agnostic Multiclass Boosting via OCO}
\end{algorithm}

Once again, the high-level framework is inspired by the work of \citet{brukhim2020online}, but like in the online setting, several pieces need to be redesigned when $k > 2$. One key difference between our algorithm and the corresponding algorithm by \citet{brukhim2020online} is the use of multiple OCO algorithms.  The algorithmic choices in Algorithm \ref{StatAgnosticBoost} are similar to those made in the online setting. Thus, the intuition provided in Section \ref{sec:oag-alg} also follows here. 

 Under the assumption that the weak learners satisfy Definition \ref{WLCS}, Theorem \ref{stataagthm}  bounds the \textit{correlation} of Algorithm \ref{StatAgnosticBoost}. 

\begin{theorem}
\label{stataagthm}
The output of Algorithm \ref{StatAgnosticBoost}, which is denoted $\bar{h}$, satisfies,

$$\E\left[cor_S(\bar{h})\right] \geq \max_{h \in \mathcal{H}} cor_S(h) - \frac{R_{\mathcal{A}}(T)}{T} - \frac{k\epsilon_0}{\gamma},$$

where $S$ is the distribution which uniformly assigns to each example $(x_i, y_i)$ probability $1/m$.

\end{theorem}

The proof of Theorem \ref{stataagthm} is split over the next two subsections.

\subsection{Improper game playing}
An important step towards proving Theorem \ref{stataagthm} is framing statistical agnostic multiclass boosting as an improper zero-sum game. A similar idea was used by \citet{brukhim2020online} for binary classification. We take a detour and elaborate on this connection here.

Consider an improper zero-sum game. In this game, there are two players A and B, and a payoff function $g$ that decomposes into the sum of $m$ smaller independent convex-concave payoff functions $f^1, ..., f^m$ each  of which depends on the players' strategies. The goal for Player A is to minimize $g$, while the goal for Player B is to maximize $g$. Let $\mathcal{K}_A$ and $\mathcal{K}_B$ be the convex, compact decisions sets of players A and B respectively. In addition, let $\mathcal{K}_C$ be a convex, compact set, and let $\mathcal{K}_{A}$ be a matrix consisting of $m$ row vectors from $\mathcal{K}_C$. Because convexity and concavtity are preserved under (non-negative weighted) summation, $g$ is also convex-concave. By Sion's minimax theorem \cite{sion1958general}, the value of the game is well-defined, which we denote as

$$\lambda^* = \min_{P \in \mathcal{K}_A} \max_{q \in \mathcal{K}_B} g(P, q) = \max_{q \in \mathcal{K}_B} \min_{P \in \mathcal{K}_a} g(P, q).$$

Let $\mathcal{K}'_B$ be a convex compact set such that $\mathcal{K}_B \subseteq \mathcal{K}'_B$. Strategies in $\mathcal{K}_B$ are proper, while those in $\mathcal{K}_B'$ are improper. We allow $g$ to be defined over $\mathcal{K}_B'$. In addition, we make the following three assumptions:

\begin{enumerate}
    \item Player B has access to a randomized approximate optimization oracle $\mathcal{W}$. Given any $P \in \mathcal{K_A}$, $\mathcal{W}$ outputs a response $q \in \mathcal{K}_B'$ such that $\E\left[g(P, q)\right] \geq \max_{q^* \in \mathcal{K}_B}g(P, q^*) - \epsilon$, where the expectation is over randomness of $\mathcal{W}$
    \item Player B is allowed to play strategies in $\mathcal{K}_B'$
    \item Player A has access to $m$ copies of possibly (independently) randomized OCO($\mathcal{K}_C$, $T$) algorithms $\mathcal{A}_1,...,\mathcal{A}_m$, all with regret $R_{\mathcal{A}}(T)$
\end{enumerate}

\smallskip
\begin{algorithm}[H]
\label{alg:gametheory}


\For{$t = 1,...,T$} {
    
    Player A plays $P_t$ \\
    Player B plays $q_t \in \mathcal{K}'_B$, where $q_t = \mathcal{W}(P_t)$ \\
    Player A and B lose/gain payoff $g(P_t, q_t) = \sum_{i=1}^{m}f^i(P_t[i], q_t)$ \\
     Define loss: $\ell^i_t(p) = f^i(p, q_t)$ $\forall i\in [m]$ \\
    Player A updates $P_{t+1}[i] = \mathcal{A}^i(\ell_1^i, ..., \ell_t^i)$ $\forall i\in [m]$ \\
    
}

\caption{Improper Game Playing}
\end{algorithm}

\begin{proposition}
\label{gametheory}
If players A and B play according to Algorithm \ref{alg:gametheory}, then player B's average strategy $\bar{q} = \frac{1}{T}\sum_{t=1}^{T}q_t, \bar{q} \in \mathcal{K}_B'$, satisfies for any $P^* \in \mathcal{K}_A$,

$$\lambda^* \leq \E\left[g(P^*, \bar{q}) \right] + \frac{mR_{\mathcal{A}}(T)}{T} + \epsilon,$$

where the expectation is over the randomness of $\mathcal{W}$.

\end{proposition}

\begin{proof}
Since the game is well-defined over $\mathcal{K}_A$ and $\mathcal{K_B}$, there exists a max-min strategy $q^* \in \mathcal{K}_B$ for player B such that for all $P \in \mathcal{K}_A$, $g(P, q^*) \geq \lambda^*$. Let $\bar{P} = \frac{1}{T}\sum_{t = 1}^{T}P_t$. Then, 

\begin{align*}
    \E \left[\frac{1}{T}\sum_{t=1}^{T}g(P_t, q_t) \right] &\geq \E\left[\frac{1}{T}\sum_{t=1}^{T}\max_{q \in \mathcal{K}_B}g(P_t, q)\right] - \epsilon \\
    &\geq \E \left[\frac{1}{T}\sum_{t=1}^{T}g(P_t, q^*) \right] - \epsilon \\
    &\geq \E\left[d(\bar{P}, q^*) \right] - \epsilon \\
    &\geq \lambda^* - \epsilon
    .
\end{align*}


Now let $\bar{q} = \frac{1}{T}\sum_{t=1}^{T}q_t$. Note that $\bar{q} \in \mathcal{K}_B'$. Then we write:

\begin{align*}
\E \left[\frac{1}{T}\sum_{t=1}^{T}g(P_t, q_t) \right] &= \E \left[\frac{1}{T}\sum_{t=1}^{T}\sum_{i = 1}^{m} f^i(P_t[i], q_t) \right] \\
&= \sum_{i=1}^{m}\E \left[\frac{1}{T}\sum_{t = 1}^{T} f^i(P_t[i], q_t) \right]\\
&\leq \sum_{i=1}^{m}\left(\E \left[\frac{1}{T}\sum_{t = 1}^{T} f^i(P^*[i], q_t)\right] + \frac{R_{\mathcal{A}}(T)}{T}\right) \\
&= \E\left[\frac{1}{T}\sum_{t=1}^{T}\sum_{i = 1}^{m} f^i(P^*[i], q_t) \right] + \frac{mR_{\mathcal{A}}(T)}{T} \\
&= \E \left[\frac{1}{T}\sum_{t=1}^{T}g(P^*, q_t) \right] + \frac{mR_{\mathcal{A}}(T)}{T} \\
&\leq \E\left[g(P^*, \bar{q}) \right] + \frac{mR_{\mathcal{A}}(T)}{T}.
\end{align*}

Here, $P^*$ is any arbitrary matrix in $\mathcal{K}_A$. Combining the lower and upper bounds, we get:

$$\lambda^* \leq \E\left[g(P^*, \bar{q}) \right] + \frac{mR_{\mathcal{A}}(T)}{T} + \epsilon,$$

which completes the proof. 
\end{proof}

\subsection{Proof of Theorem \ref{stataagthm}}

Now we are ready to prove Theorem \ref{stataagthm}. The proof strategy will be as follows. We will first show how the proposed agnostic boosting algorithm is an instance of the improper game playing setup described above. Then, we will show that a weak learner satisfying Definition \ref{WLCS} corresponds to the randomized approximate optimization oracle. Finally, we will explicitly compute the value of the game, and derive the lower bound on correlation. 

Under Definition \ref{WLCS}, we can carefully argue that the agnostic weak learner $\mathcal{W}$ induces an approximate optimization oracle for Player B. Specifically, we can show the following lemma. 

\begin{lemma}
For any $P \in \mathcal{K}_A$, the output $q' = \frac{\mathcal{W}(P)}{\gamma}\in \mathcal{K}'_B$ satisfies, 

$$\E\left[g(P, q')\right] \geq \max_{q \in \mathcal{K}_B}g(P, q) - \frac{mk\epsilon_0}{\gamma}.$$ 
\end{lemma}

\begin{proof}
    
Note that in line $3$ of Algorithm \ref{StatAgnosticBoost}, we pass re-labelled examples back to the weak learner. An alternative and equivalent approach to what is presented is to first relabel each example using $P$ and then to uniformly sample and pass $m_0$ examples to $\mathcal{W}$. Let $\tilde{y}_i$ correspond to the relabeled class of example $i$. Recall that $\tilde{y}_i \sim P[i]$. Let $R$ be the distribution over $(x_i, \tilde{y}_i)$ after relabelling and define $h^* = \max_{h \in \mathcal{H}} cor_R(h)$. Then, by the weak learning assumption, 

\begin{align*}
    \E_{S'|\tilde{y}}\left[\sum_{i=1}^m  \mathcal{W}_{S'}(x_i) \cdot (2 \tilde{y}_i - \mathbbm{1}_k)\right] &= 
    \E_{S'|\tilde{y}}\left[\sum_{\ell \in \mathcal{B}_k} \sum_{i: \tilde{y}_i = \ell}  \mathcal{W}_{S'}(x_i) \cdot (2 \tilde{y}_i - \mathbbm{1}_k)\right] \\
    & \geq \sum_{\ell \in \mathcal{B}_k}\left(\gamma \max_{h \in \mathcal{H}}\sum_{i:\tilde{y}_i = \ell} h(x_i) \cdot (2\tilde{y}_i - \mathbbm{1}_k)  - m\epsilon_0\right)\\
    & \geq \gamma\sum_{\ell \in \mathcal{B}_k}\sum_{i:\tilde{y}_i = \ell} h^*(x_i) \cdot (2\tilde{y}_i - \mathbbm{1}_k)  - mk\epsilon_0\\
    &= \gamma\sum_{i=1}^m h^*(x_i) \cdot (2\tilde{y}_i - \mathbbm{1}_k)  - mk\epsilon_0.\\
\end{align*}

Taking the expectation of both sides, 

\begin{align*}
    \E_{\tilde{y}} \left[\E_{S'|\tilde{y}}\left[\sum_{i=1}^m  \mathcal{W}_{S'}(x_i) \cdot (2 \tilde{y}_i - \mathbbm{1}_k)\right]\right] &= \E\left[\sum_{i=1}^m  \mathcal{W}_{S'}(x_i) \cdot (2 \tilde{y}_i - \mathbbm{1}_k)\right]\\
    &=  \E_{\tilde{S'}} \left[\E_{\tilde{y}|S'}\left[\sum_{i=1}^m  \mathcal{W}_{S'}(x_i) \cdot (2 \tilde{y}_i - \mathbbm{1}_k)\right]\right] \\
    & \geq  \E_{\tilde{y}}\left[\gamma\sum_{i=1}^m h^*(x_i) \cdot (2\tilde{y}_i - \mathbbm{1}_k) \right]- mk\epsilon_0. \\
\end{align*}

 For any $h \in \mathcal{H}$,

\begin{align*}
    \E_{\tilde{y}}\left[ \frac{1}{m}\sum_{i=1}^m  h(x_i) \cdot (2\tilde{y}_i - \mathbbm{1}_k)\right] &= \E_{(x_i, y_i) \sim R}\left[ h(x_i) \cdot (2y_i - \mathbbm{1}_k)\right]\\
    &= 
    \E_{x_i}\left[\E_{y_i|x_i}\left[h(x_i) \cdot (2y_i-\mathbbm{1}_k)\right] \right]\\
    &= \E_{x_i}\left[h(x_i) \cdot (2 \E\left[y_i|x_i\right]-\mathbbm{1}_k) \right]\\
    &= \E_{x_i}\left[h(x_i) \cdot (2P[i]-\mathbbm{1}_k) \right]\\
    &= \frac{1}{m}\sum_{i=1}^{m}h(x_i) \cdot (2P[i]-\mathbbm{1}_k)\\
    &= \E_{(x_i, y_i) \sim D}\left[ h(x_i) \cdot (2P[i] - \mathbbm{1}_k)\right].
\end{align*}

Note the last equality follows from the fact that both $R$ and $D$ have the same marginal distribution over unlabeled examples $x_i$'s. Combining the results, multiplying by $m$, and dividing by $\gamma$ gives, 

 $$ \E_{S'}\left[\sum_{i=1}^{m}\frac{\mathcal{W_{S'}}(x_i)}{\gamma}\cdot (2P[i]-\mathbbm{1}_k)\right] \geq \sum_{i=1}^{m}h^*(x_i) \cdot (2P[i]-\mathbbm{1}_k) - \frac{mk\epsilon_0}{\gamma}.$$
 
 Finally, note that $q'(x_i) = \frac{\mathcal{W}_{S'}(x_i)}{\gamma} \in \mathcal{K}_B^{'}$. Then, recalling our definition of $g$, 
 
 \begin{align*}
     \E\left[g(P, q')\right] &= \E\left[\sum_{i = 1}^{m}(2P[i] - \mathbbm{1}_k)\cdot(q'(x_i) - y_i)\right]\\
     &= \E\left[\sum_{i = 1}^{m}(2P[i] - \mathbbm{1}_k)\cdot q'(x_i)\right] - \sum_{i = 1}^{m}(2P[i] - \mathbbm{1}_k)\cdot y_i\\
     &\geq \sum_{i=1}^{m}h^*(x_i) \cdot (2P[i]-\mathbbm{1}_k) - \frac{mk\epsilon_0}{\gamma} - \sum_{i = 1}^{m}(2P[i] - \mathbbm{1}_k)\cdot y_i\\
     &= \sum_{i=1}^{m}(h^*(x_i) - y_i) \cdot (2P[i]-\mathbbm{1}_k) - \frac{mk\epsilon_0}{\gamma}\\
     &= \max_{q \in \mathcal{K}_B}g(P, q) - \frac{mk\epsilon_0}{\gamma}.
 \end{align*}
 
 which completes the proof.

\end{proof}

Now, we return to proving Theorem \ref{stataagthm}. We start by explicitly computing the value of the above game. One can show that the dominant strategy for player A is to return a matrix where row $i$ is the vector $y_i$. That is, $P = \left[y_1; y_2; ...; y_m\right]$. Because $g$ decomposes into the sum of smaller independent losses $f$, it is helpful to instead focus our attention on 

$$f^i(P[i], q) = (2P[i] - \mathbbm{1}_k)\cdot(q(x_i) - y_i).$$

Above, Player A has control over the vector $P[i]$ and wants to minimize $f^i$. We can show that regardless of what $q(x_i) \in \mathcal{K}_B$ is, Player A minimizes $f$ by playing $P[i] = y_i$. Note, only the case where $q(x_i) \neq y_i$ is important. Under this scenario, Player A needs to maximize the value at the $y_i$th index and minimize the value at the $q(x_i)$th index. This is precisely accomplished by setting $P[i] = y_i$. Since the smaller loss functions $f^i$ are independent of one another, it follows that Player A minimizes $g$ by playing matrix $P = \left[y_1; y_2; ...; y_m\right]$. Under this fixed dominant strategy for A, the value of the game $\lambda^*$ can be computed as 

\begin{align*}
    \lambda^* &= \max_{q \in \mathcal{K}_B} \min_{P \in \mathcal{K}_a} g(P, q)\\
    &= \max_{q \in \mathcal{K}_B} \sum_{i = 1}^{m}(2y_i - \mathbbm{1}_k)\cdot(q(x_i) - y_i)\\
    &= \max_{q \in \mathcal{K}_B}\sum_{i = 1}^{m} 2(q(x_i) \cdot y_i - 1)\\
    &= m \cdot cor_S(h^*) - m.
\end{align*}

Then, for any $P^*$ using Proposition \ref{gametheory},  

\begin{align*}
m \cdot \max_{h \in \mathcal{H}} cor_S(h) - m &\leq \E\left[g(P^*, \bar{q}) \right] + \frac{mR_{\mathcal{A}}(T)}{T} + \frac{mk\epsilon_0}{\gamma} \\
&= \E\left[ \sum_{i = 1}^{m}(2P^*[i] - \mathbbm{1}_k)\cdot(\bar{q}(x_i) - y_i) \right] + \frac{mR_{\mathcal{A}}(T)}{T} + \frac{mk\epsilon_0}{\gamma} \\
&= \E\left[ \sum_{i = 1}^{m}(2P^*[i] - \mathbbm{1}_k)\cdot(\frac{1}{\gamma T} \sum_{t=1}^{T}h_t(x_i) - y_i) \right] + \frac{mR_{\mathcal{A}}(T)}{T} + \frac{mk\epsilon_0}{\gamma}\\
&= \E\left[ \sum_{i = 1}^{m}P^*[i]\cdot\left(\frac{\frac{2}{T}\sum_{t=1}^Th_t(x_i) - \mathbbm{1}_k}{\gamma} - (2y_i - \mathbbm{1}_k)\right) \right] \\
&= \quad + \frac{mR_{\mathcal{A}}(T)}{T} + \frac{mk\epsilon_0}{\gamma}. \quad\quad \text{(Lemma \ref{equivalence})}
\end{align*} \\

According to Lemma \ref{upperbound}, there exists a $P^*[i]$ and therefore a $P^*$, such that 

\begin{align*}
    m \cdot \max_{h \in \mathcal{H}} cor_S(h) - m &\leq \E\left[ \sum_{i = 1}^{m}2\left( \prod\left(\frac{\sum_{i=1}^Th_t(x_i)}{\gamma T}\right) \cdot y_i - 1\right)\right] + \frac{mR_{\mathcal{A}}(T)}{T} + \frac{mk\epsilon_0}{\gamma}\\
    &= \E\left[ \sum_{i = 1}^{m}2\bar{h}(x_t)\cdot y_i - 1\right] - m + \frac{mR_{\mathcal{A}}(T)}{T} + \frac{mk\epsilon_0}{\gamma}.
\end{align*}

Subtracting and then dividing both sides by $m$,

\begin{align*}
    \max_{h \in \mathcal{H}} cor_S(h) &\leq \E\left[ \frac{1}{m}\sum_{i = 1}^{m}2\bar{h}(x_t)\cdot y_i - 1\right]  + \frac{R_{\mathcal{A}}(T)}{T} + \frac{k\epsilon_0}{\gamma}\\
    &= \E\left[cor_S(\bar{h}) \right]  + \frac{R_{\mathcal{A}}(T)}{T} + \frac{k\epsilon_0}{\gamma}.
\end{align*}

Rearranging, we have shown that

$$\E\left[cor_S(\bar{h}) \right] \geq \max_{h \in \mathcal{H}} cor_S(h) - \frac{R_{\mathcal{A}}(T)}{T} - \frac{k\epsilon_0}{\gamma}$$

which completes the proof. If one further lets $R_{\mathcal{A}}(T)$ to be OGD, as in the online setting, then observe that we get

$$\E\left[cor_S(\bar{h}) \right] \geq \max_{h \in \mathcal{H}} cor_S(h) - \frac{R_{\mathcal{A}}(T)}{T} - \frac{k\epsilon_0}{\gamma}$$

\section{Realizable Multiclass Boosting}
\label{AppendixRealBoost}

In this section, we fill in the gaps of Section \ref{sec:beyond} for the online and statistical realizable settings by providing boosting algorithms and their theoretical analysis. Namely, we show explicitly how the OCO framework can also be used to construct multiclass boosting algorithms in the realizable setting. Unlike in the agnostic setting, our realizable boosting algorithms will update weak learners by \textit{reweighting} examples.
 
 \subsection{Online Setting}
 We now describe a online realizable boosting algorithm whose pseudocode is provided below. The booster has black-box oracle access to $N$ copies of a $(\gamma, T)$-RWOL algorithm
 $\mathcal{W}_1, ..., \mathcal{W}_N$, each satisfying Definition \ref{WLCOR}, and an $\text{OCO}([0, 1], N)$ algorithm $\mathcal{A}$.  Importantly, note that in line 4 , we denote by $\prod$ the $L_2$ projection onto $\Delta_k$.
 
 \smallskip
\begin{algorithm}[H]
\label{ORB}

\KwIn{$(\gamma, T)$-RWOL $\mathcal{W}_1...\mathcal{W}_N$, $\text{OCO}([0, 1], N)$ algorithm $\mathcal{A}$}

\For{$t = 1,...,T$} {

    Receive example $x_t$ \\
    Accumulate weak predictions $h_t = \sum_{i=1}^N\mathcal{W}_i(x_t)$\\
    Set $\mathcal{D}_t = \Pi(\frac{h_t}{\gamma N})$\\
    Predict $\hat{y_t} \sim \mathcal{D}_t$ \\
    Receive true label $y_t$ \\
    \For{$i = 1,...,N$} {
        If $i > 1$, obtain $p_t^i = \mathcal{A}(l_t^1,...,l_t^{i-1})$. Else, initialize $p^1_t = 0.5$. \\
        Reveal loss function: $l_t^i(p) = p(\frac{2\mathcal{W}_i(x_t)\cdot y_t - 1}{\gamma} - 1)$ \\
        Pass $(x_t, y_t)$ to $\mathcal{W}_i$ w.p. $p_t^i$\\
    }
Reset $\mathcal{A}$
}
\caption{Online Realizable Multiclass Boosting via OCO}
\end{algorithm}

\begin{theorem}
The regret bound of Algorithm \ref{ORB} satisfies:
$$\frac{1}{T}\sum_{t=1}^T (2\E\left[\hat{y}_t\right] \cdot y_t - 1)  \geq 1 - \frac{R_{\mathcal{A}}(N)}{N} -\frac{\tilde{R}_{\mathcal{W}}(T, k)}{\gamma T},$$

where $\tilde{R}_{\mathcal{W}}(T, k) = 2R_{\mathcal{W}}(T, k) + \tilde{O}(\sqrt{T})$.

\end{theorem}

\begin{proof}

As usual, the approach is to compute a lower and upper bound on the sum of the expected losses passed to the OCO oracle. Starting with the upper bound,

    \begin{align*}
        \E\left[\frac{1}{N}\sum_{i=1}^{N} l_t^i(p_t^i)\right] &\leq \frac{1}{N}\E\left[\min_{p \in \mathcal{K}}\sum_{i=1}^{N}l^i_t(p)\right] + \frac{R_{\mathcal{A}}(N)}{N} && \text{(OCO Regret)}\\
        &\leq \frac{1}{N}\E\left[\sum_{i=1}^{N}l^i_t(p^*)\right] + \frac{R_{\mathcal{A}}(N)}{N} \\
        &= \frac{1}{N}\sum_{i=1}^{N}\E\left[p^*\left(\frac{2\mathcal{W}_i(x_t)\cdot y_t - 1}{\gamma} - 1\right)\right] + \frac{R_{\mathcal{A}}(N)}{N} \\
        &= \E\left[p^*\left( \frac{2\left(\frac{1}{N}\sum_{i=1}^N\mathcal{W}_i(x_t)\right)\cdot y_t - 1}{\gamma} - 1) \right) \right] + \frac{R_{\mathcal{A}}(N)}{N} \\
        &\leq \E\left[2\left(\prod\left(\frac{1}{\gamma N}\sum_{i=1}^N \mathcal{W}_i(x_t)\right) \cdot y_t - 1\right)\right]  + \frac{R_{\mathcal{A}}(N)}{N} && \text{(Lemma \ref{upperbound_real})}\\\\
        &= 2\left(\E\left[\hat{y}_t\right] \cdot y_t - 1 \right) + \frac{R_{\mathcal{A}}(N)}{N}.\\
    \end{align*}

Finally, summing over $t \in [T]$, 

$$\E\left[\frac{1}{N}\sum_{t=1}^T\sum_{i=1}^{N} l_t^i(p_t^i)\right] \leq \sum_{t=1}^T2\left(\E\left[\hat{y}_t\right] \cdot y_t - 1 \right) + \frac{TR_{\mathcal{A}}(N)}{N}.$$

For the lower bound, we first need the following important Lemma adapted from \cite{beygelzimer2015optimal} to the multiclass setting:

\begin{lemma}
\label{notmylemma}
For any weak learner $(\gamma, T)$-RWOL $\mathcal{W}$, there exists $\tilde{R}_{\mathcal{W}}(T, k) = \tilde{\mathcal{O}}(\sqrt{\sum_t p_t}) + 2R_{\mathcal{W}}(T, k)$ such that for any sequence $p_1, ... ,p_T \in [0, 1]$,

$$\sum_{t=1}^{T} p_t (2\mathcal{W}(x_t)\cdot y_t - 1) \geq \gamma \sum_{t=1}^T p_t - \tilde{R}_{\mathcal{W}}(T, k).$$
\end{lemma}

The proof of Lemma \ref{notmylemma} follows exactly from Lemma 14 in \citet{brukhim2020online} and Lemma 1 in \citet{beygelzimer2015optimal} and so we omit it here. Using Lemma \ref{notmylemma}, 

\begin{align*}
    \frac{1}{\gamma}\E\left[ \sum_{i=1}^N \sum_{t=1}^T p_t^i(2\mathcal{W}(x_t)\cdot y_t - 1) \right] &\geq \E\left[ \frac{1}{\gamma}\sum_{i=1}^N \left(\gamma\sum_{t=1}^Tp_t^i - \tilde{R}_{\mathcal{W}}(T, k) \right) \right]\\
    &= \sum_{i=1}^N \sum_{t=1}^T\E \left[p_t^i\right] - \frac{N}{\gamma}\tilde{R}_{\mathcal{W}}(T, k).
\end{align*}

Then, 

$$\E\left[\sum_{t=1}^T \sum_{i=1}^N l_t^i(p_t^i) \right] =  \frac{1}{\gamma}\E\left[ \sum_{i=1}^N \sum_{t=1}^T p_t^i(2\mathcal{W}(x_t)\cdot y_t - 1) \right] - \sum_{i=1}^N \sum_{t=1}^T\E \left[p_t^i\right] \geq -\frac{N}{\gamma}\tilde{R}_{\mathcal{W}}(T, k).$$

By combining upper and lower bounds for $\E\left[\frac{1}{NT}\sum_t \sum_i l_t^i(p_t^i) \right]$, we get

$$\frac{1}{T}\sum_{t=1}^T (2\E\left[\hat{y}_t\right] \cdot y_t - 1)  \geq 1 - \frac{R_{\mathcal{A}}(N)}{N} -\frac{\tilde{R}_{\mathcal{W}}(T, k)}{\gamma T}.$$

which completes the proof. 
\end{proof}

 \subsection{Statistical Setting}
 We now describe a statistical realizable boosting algorithm whose pseudocode is provided below. The booster is given as input a sample $S = (x_1, y_1), ..., (x_m, y_m) \in \mathcal{X} \times \mathcal{B}_k$. The booster has black-box oracle access to $T$ copies of a $(\gamma, m_0)$-RWL algorithm
 $\mathcal{W}_1, ..., \mathcal{W}_T$, each satisfying Definition \ref{WLCSR}, and $m$ copies of an $\text{OCO}([0, 1], T)$ algorithm, $\mathcal{A}^1, ..., \mathcal{A}^m$.  Importantly, note that in line 8 , we denote by $\prod$ as a randomized \textit{prediction} operator that given a $\gamma$-scaled distribution over $k$ labels, first computes the $L_2$ projection onto $\Delta_k$, and then randomly samples a label from the resulting distribution.
 
\smallskip
\begin{algorithm}[H]
\label{alg:statrealizable}
\KwIn{$(\gamma, m_0)$-RWLs $\mathcal{W}_1...\mathcal{W}_T$, OCO($[0, 1]$, $T$) algorithms $\mathcal{A}^1...\mathcal{A}^m$, and Sample $S = (x_1, y_1)...(x_m, y_m) \in \mathcal{X} \times \mathcal{B}_k$}

\For{$t = 1,...,T$} {
    If $t = 1$, set $P_1[i] = 1/m$ $\forall i\in [m]$\\
    Pass $m_0$ examples to $\mathcal{W}_t$ drawn from the following distribution: \\
    Draw $(x_i, y_i)$ w.p. $\propto P_{t}[i]$  \\
    Let $h_t$ be the weak hypothesis returned by $\mathcal{W}_t$ \\
    Reveal loss function: $l_t^i(p) = p(\frac{h_t(x_i)}{\gamma}\cdot(2y_i - \mathbbm{1}_k) - 1)$ $\forall i\in[m]$ \\
    Set $P_{t+1}[i] =\mathcal{A}^i(l_1^i, ..., l_t^i)$ $\forall i \in [m]$
}

return $\bar{h}(x) = \prod\left(\frac{1}{\gamma T} \sum_{t = 1}^{T} h_t(x)\right)$
\caption{Statistical Realizable Boosting via OCO}
\end{algorithm}

\smallskip

\begin{theorem}
The output of Algorithm \ref{alg:statrealizable}, which is denoted $\bar{h}$, satisfies, 


$$\E\left[cor_S(\bar{h})\right] \geq 1 - \frac{R_{\mathcal{A}}(T)}{T},$$

where $S$ is the distribution which uniformly assigns to each example $(x_i, y_i)$ probability $1/m$.

\end{theorem}

\begin{proof}
We will use a very similar proof strategy as in the agnostic setting by reducing to improper game playing. Let $h^*$ be a hypothesis consistent with the input sample (i.e. $h^*(x_i) = y_i$ for
$ i \leq m$) and let $\mathcal{H'} = \mathcal{H} \cup h^*$. We begin by establishing the reduction to improper game playing. Take $\mathcal{K}_A = [0, 1]^m$, $\mathcal{K}_B = \Delta_{\mathcal{H}}$, $\mathcal{K}'_B = \frac{1}{\gamma}\mathcal{K}_B$.  Define payoff functions:

$$f^i(p, q) = p(q(x_i) \cdot (2y_i - \mathbbm{1}_k)  - 1)$$

\begin{align*}
g(P, q) &= \sum_{i=1}^{m}f^i(P[i], q)\\
&= \sum_{i = 1}^{m}P[i](q(x_i) \cdot (2y_i - \mathbbm{1}_k)  - 1).
\end{align*}

We will now show that the agnostic weak learner $\mathcal{W}$ induces an approximate optimization oracle for Player B. Specifically, we will show the following lemma. 

\begin{lemma}
For any $P \in \mathcal{K}_A$, the output $q' = \frac{\mathcal{W}(P)}{\gamma}\in \mathcal{K}'_B$ satisfies 

$$\E\left[g(P, q')\right] \geq 0.$$
\end{lemma}

That is, the weak learner corresponds to an approximate oracle with $0$ error. 

\begin{proof}
Using the definition of correlation and the weak learning condition. 
$$
    \E\left[cor_P(\mathcal{W_{S'}})\right] = \E\left[\sum_{i=1}^{m} \frac{P[i]}{\sum_{i=1}^{m} P[i]}\mathcal{W_{S'}}(x_i)\cdot (2y_i-\mathbbm{1}_k)\right]\\
    \geq \gamma
$$

and therefore, 

$$
   \E\left[\sum_{i=1}^{m} P[i]\frac{\mathcal{W_{S'}}(x_i)}{\gamma}\cdot (2y_i-\mathbbm{1}_k)\right]\\
    \geq \sum_{i=1}^{m} P[i].
$$

Now, take $q'(x_i) = \frac{\mathcal{W}_{S'}(x_i)}{\gamma} \in \mathcal{K}_B^{'}$. Then, recalling our definition of $g$, 
 
 \begin{align*}
     \E\left[g(P, q')\right] &= \E\left[\sum_{i = 1}^{m}P[i](q'(x_i) \cdot (2y_i - \mathbbm{1}_k)  - 1)\right]\\
     &= \E\left[\sum_{i = 1}^{m} P[i]q'(x_i)\cdot(2y_i - \mathbbm{1}_k) \right] - \sum_{i=1}^m P[i]\\
     &\geq 0 = \max_{q \in \mathcal{K}_B}g(P, q).\\
 \end{align*}
\end{proof}

Now, we will complete the last part of the overall proof. We start by explicitly computing the value of the above game. One can show that the dominant strategy for player B is to return $h^*$. Indeed, since $h^*$ is consistent, we can show that $g(P, h^*) = 0$ for any $P$ and thus $\lambda^* = 0$.
Then, for any $P^*$ using Proposition \ref{gametheory},  

\begin{align*}
0 &\leq \E\left[g(P^*, \bar{q}) \right] + \frac{mR_{\mathcal{A}}(T)}{T} \\
&= \E\left[ \sum_{i = 1}^{m}P^*[i](\bar{q} \cdot (2y_i - \mathbbm{1}_k)  - 1)\right] + \frac{mR_{\mathcal{A}}(T)}{T}\\
&= \E\left[ \sum_{i = 1}^{m}P^*[i]\left(\left(\frac{1}{\gamma T} \sum_{t=1}^{T}h_t(x_i)\right) \cdot (2y_i - \mathbbm{1}_k)  - 1\right)\right] + \frac{mR_{\mathcal{A}}(T)}{T}\\
&= \E\left[ \sum_{i = 1}^{m}P^*[i]\left(\frac{\left(\frac{2}{T}\sum_{t=1}^Th_t(x_i)\right)\cdot y_i - 1}{\gamma} - 1\right) \right] + \frac{mR_{\mathcal{A}}(T)}{T}.
\end{align*} \\

According to Lemma \ref{upperbound_real}, there exists a $P^*[i]$ and therefore a $P^*$, such that 

\begin{align*}
    0 &\leq \E\left[ \sum_{i = 1}^{m}2\left( \prod\left(\frac{\sum_{i=1}^Th_t(x_i)}{\gamma T}\right) \cdot y_i - 1\right)\right] + \frac{mR_{\mathcal{A}}(T)}{T}\\
    &= \E\left[ \sum_{i = 1}^{m}2\bar{h}(x_t)\cdot y_i - 1\right] - m + \frac{mR_{\mathcal{A}}(T)}{T}.
\end{align*}

 Dividing both sides by $m$,

\begin{align*}
    0 &\leq \E\left[ \frac{1}{m}\sum_{i = 1}^{m}2\bar{h}(x_t)\cdot y_i - 1\right]  - 1 + \frac{R_{\mathcal{A}}(T)}{T}\\
    &= \E\left[cor_S(\bar{h}) \right] - 1 + \frac{R_{\mathcal{A}}(T)}{T}.
\end{align*}

Thus, rearranging, we have shown that

$$\E\left[cor_S(\bar{h}) \right] \geq 1 - \frac{R_{\mathcal{A}}(T)}{T}.$$
\end{proof}

\section{Alternate Weak Learning Conditions}
\label{AppendixAltWL}

Definition $\ref{WLC}$ roughly requires that the online weak learner be able to distinguish between every pair of classes to some non-trivial, but far from optimal, degree. In this section, we propose other possible online agnostic weak learning conditions \textit{without} changing Algorithm \ref{OnlineAgnosticBoostFull}.

We start by first considering a setting where there always exists a hypothesis $h \in \mathcal{H}$ that, in expectation, makes at most $T/2$ mistakes (i.e. gets at least $1/2$ correct). Here, the weak learning condition below suffices as  $\max_{h \in \mathcal{H}} \E\left[ \sum_{t = 1}^T 2h(x_t)\cdot y_t - 1 \right]$ is guaranteed to be positive. 

\begin{definition}
 \label{WLC2}
 Let $\mathcal{H} \subseteq \mathcal{B}_k^{\mathcal{X}}$ be a class of experts and let $0 < \gamma \leq 1$ denote the \say{advantage}. An online learning algorithm $\mathcal{W}$ is a $(\gamma, T)$-agnostic weak online learner for $\mathcal{H}$ if for any adaptively chosen sequence of tuples $(x_t, y_t) \in \mathcal{X} \times \mathcal{B}_k$, at every iteration $t \in [T]$, the algorithm outputs $\mathcal{W}(x_t) \in \mathcal{B}_k$ such that, 
 
  \[\gamma \max_{h \in \mathcal{H}} \E\left[ \sum_{t = 1}^T 2h(x_t)\cdot y_t - 1 \right] - \E\left[ \sum_{t=1}^T 2\mathcal{W}(x_t) \cdot y_t - 1\right] \leq R_\mathcal{W}(T, k),\]
  
   where the expectation is taken w.r.t. the randomness of the weak learner $\mathcal{W}$ and that of the possibly adaptive
adversary, $R_\mathcal{W} : \mathbbm{N} \times \mathbbm{N} \rightarrow \mathbbm{R}_+$ is the additive regret: a non-decreasing, sub-linear function of $T$.
  \end{definition}
  
  For weak learners satisfying Definition \ref{WLC2}, one can show, using the same proof strategy in Section \ref{sec:oag-reg}, that Algorithm \ref{OnlineAgnosticBoostFull} achieves the average regret bound below.
  
  \begin{proposition}
   Assuming weak learners satisfy Definition \ref{WLC2}, the expected regret bound of Algorithm \ref{OnlineAgnosticBoostFull} is
   
   \[\frac{1}{T}\max_{h \in \mathcal{H}}\E\left[ \sum_{t=1}^{T} 2h(x_t) \cdot y_t - 1 \right] - \frac{1}{T}\E\left[\sum_{t=1}^{T}2\hat{y}_t \cdot y_t - 1\right] \leq \frac{R_{\mathcal{W}}(T, k)}{\gamma T} + \frac{R_{\mathcal{A}}(N)}{N}.\]

  \end{proposition}

  For completeness sake, we include a partial proof below. 
  
  \begin{proof}
      The proof follows the same strategy as that used to prove Theorem \ref{thm1}. That is, we will upper and lower bound the expected sum of loss passed to the OCO algorithm. Fortunately, the proof of the upper bound remains exactly the same as that in the proof of Theorem \ref{thm1}. Thus, we will only derive a lower bound on the sum of losses passed to the OCO algorithm. Define  $h^* = \argmax_{h \in \mathcal{H}}\sum_{t=1}^T  (2h(x_t)\cdot y_t - 1)$ as the optimal competitor in hindsight. Then, 
      
      \begin{align*}
        \E\left[\sum_{i=1}^{N} \sum_{t=1}^{T} l_t^i(p_t^i)\right] &= \E\left[\sum_{i=1}^{N} \sum_{t=1}^{T} p_t^i \cdot \left(\frac{2\mathcal{W}_i(x_t) - \mathbbm{1}_k}{\gamma} - (2y_t - \mathbbm{1}_k) \right)\right] \\
        &= \frac{1}{\gamma}\sum_{i=1}^{N} \sum_{t=1}^{T}\E\left[p_t^i \cdot (2\mathcal{W}_i(x_t) - \mathbbm{1}_k)\right]   - \sum_{i=1}^{N} \sum_{t=1}^{T} \E\left[p_t^i \cdot (2y_t - \mathbbm{1}_k)\right] \\
        &= \frac{1}{\gamma}\sum_{i=1}^{N} \sum_{t=1}^{T}\E\left[ 2\mathcal{W}_i(x_t) \cdot y_t^i - 1\right]   - \sum_{i=1}^{N} \sum_{t=1}^{T} \E\left[p_t^i \cdot (2y_t - \mathbbm{1}_k)\right] && \text{(Lemma \ref{expectation})}\\
    \end{align*}
    
Using the weak learning condition in Definition \ref{WLC2},

    \begin{align*}
       \frac{1}{\gamma}\sum_{i=1}^{N} \sum_{t=1}^{T}\E\left[ 2\mathcal{W}_i(x_t) \cdot y_t^i - 1\right] &\geq \sum_{i=1}^{N} \max_{h \in \mathcal{H}}\sum_{t=1}^{T}\E\left[ 2h(x_t) \cdot y_t^i - 1\right] - \frac{NR_{\mathcal{W}}(T, k)}{\gamma} && \text{(Definition \ref{WLC2})}\\
        &\geq \sum_{i=1}^{N} \sum_{t = 1}^T\E\left[ 2h^*(x_t) \cdot y_t^i - 1)\right]  - \frac{NR_{\mathcal{W}}(T, k)}{\gamma}\\
        &= \sum_{i=1}^{N}\sum_{t=1}^{T}\E\left[ p_t^i \cdot (2h^*(x_t) - \mathbbm{1}_k)\right] - \frac{NR_{\mathcal{W}}(T, k)}{\gamma}. && \text{(Lemma \ref{expectation})}
    \end{align*}
    
Putting things together, we find, 
    \begin{align*}
        \E\left[\sum_{i=1}^{N} \sum_{t=1}^{T} l_t^i(p_t^i)\right] &\geq \sum_{i=1}^{N}\sum_{t=1}^{T}\E\left[ p_t^i \cdot (2h^*(x_t) - 2y_t) \right] - \frac{NR_{\mathcal{W}}(T, k)}{\gamma}\\
        &\geq \sum_{i=1}^{N}\sum_{t=1}^{T}\E\left[ 2(h^*(x_t) \cdot y_t - 1) \right] - \frac{NR_{\mathcal{W}}(T, k)}{\gamma} && \text{(Lemma \ref{lowerbound})}\\
        &= N\sum_{t=1}^{T} 2(h^*(x_t) \cdot y_t - 1)  - \frac{NR_{\mathcal{W}}(T, k)}{\gamma}.\\
    \end{align*}

Note that this is the same lower bound as in the proof of Theorem \ref{thm1} and therefore the same regret bound as in Theorem \ref{thm1} holds. 
  \end{proof}
  
The assumption on the hypothesis class that  there exists a hypothesis $h \in \mathcal{H}$ that, in expectation, makes at most $T/2$ mistakes is quite strong, especially if we allow a random adaptive adversary. Note that Definition \ref{WLC2} cannot be used when this assumption on $\mathcal{H}$ does not hold - $\E\left[ \sum_{t=1}^{T} 2h(x_t) \cdot y_t - 1 \right]$ can be potentially negative for every $h \in \mathcal{H}$,  even for the randomly guessing competitor. One way to relax this assumption on $\mathcal{H}$ is via Definition \ref{WLC} presented in the main text. Another way, is by taking a \textit{one-vs-all} perspective to multiclass classification, which we present below. 

 \begin{definition}
 \label{WLC3}
 Let $\mathcal{H} \subseteq \mathcal{B}_k^{\mathcal{X}}$ be a class of experts and let $0 < \gamma \leq 1$ denote the \say{advantage}. An online learning algorithm $\mathcal{W}$ is a $(\gamma, T)$-agnostic weak online learner for $\mathcal{H}$ if for any adaptively chosen sequence of tuples $(x_t, y_t) \in \mathcal{X} \times \mathcal{B}_k$, at every iteration $t \in [T]$, for every label $\ell \in \mathcal{B}_k$, the algorithm outputs $\mathcal{W}(x_t) \in \mathcal{B}_k$ such that, 
  \[\gamma \max_{h \in \mathcal{H}} \E\left[ \sum_{t:y_t = \ell} 2h(x_t)\cdot y_t - 1 \right] - \E\left[ \sum_{t: y_t = \ell} 2\mathcal{W}(x_t) \cdot y_t - 1\right] \leq R_\mathcal{W}(T, k),\]
 where the expectation is taken w.r.t. the randomness of the weak learner $\mathcal{W}$ and that of the possibly adaptive
adversary, $R_\mathcal{W} : \mathbbm{N} \times \mathbbm{N} \rightarrow \mathbbm{R}_+$ is the additive regret: a non-decreasing, sub-linear function of $T$.
 \end{definition}
 
Intuitively, Definition \ref{WLC3} roughly requires that a weak learner be able to learn each class to a non-trivial (but far from optimal) degree of accuracy as it processes the stream of examples. 
In order to ensure that $\gamma \in (0, 1)$, we also need that our hypothesis class $\mathcal{H}$ contains $k$ hypothesis each of which predicts a reference class $\ell \in \mathcal{B}_k$ with probability a $1/2$. This assumption on the hypothesis class is mild: if our class $\mathcal{H}$ does not contain these randomly guessing binary hypotheses, we can add them without substantially increasing the complexity of the class $\mathcal{H}$. That said, this is still a minor drawback compared to Definition \ref{WLC}, where no additional assumptions on $\mathcal{H}$ are needed. 
 
 Under Definition \ref{WLC3}, we can show that the following  regret bound for Algorithm \ref{OnlineAgnosticBoostFull}. Note that this regret bound is worse compared to the one derived assuming the weak learners satisfy Definition \ref{WLC}. Indeed, compared to Theorem \ref{thm1}, there is an extra factor of $k$ in the term containing $R_{\mathcal{W}(T, k)}$ in the bound below.
 
   \begin{proposition}
   Assuming weak learners satisfy Definition \ref{WLC3}, the expected regret bound of Algorithm \ref{OnlineAgnosticBoostFull} is
   
    \[\frac{1}{T}\max_{h \in \mathcal{H}}\E\left[ \sum_{t=1}^{T} 2h(x_t) \cdot y_t - 1 \right] - \frac{1}{T}\E\left[\sum_{t=1}^{T}2\hat{y}_t \cdot y_t - 1\right] \leq \frac{kR_{\mathcal{W}}(T, k)}{\gamma T} + \frac{R_{\mathcal{A}}(N)}{N}.\]

  \end{proposition}
  
  \begin{proof}
      The proof follows the same strategy as that used to prove Theorem \ref{thm1}. Like above, the proof of the upper bound remains exactly the same as that in the proof of Theorem \ref{thm1}. Thus, we will only derive a lower bound on the sum of losses passed to the OCO algorithm. Let   $h^* = \argmax_{h \in \mathcal{H}}\E\left[\sum_{t=1}^T  2h(x_t)\cdot y_t - 1\right]$ and $h^*_{i\ell} = \argmax_{h \in \mathcal{H}}\E\left[\sum_{t: y_t^i = \ell}^T 2h(x_t) \cdot y_t^i - 1 \right]$. Then, 
    \begin{align*}
        \E\left[\sum_{i=1}^{N} \sum_{t=1}^{T} l_t^i(p_t^i)\right] &= \E\left[\sum_{i=1}^{N} \sum_{t=1}^{T} p_t^i \cdot \left(\frac{2\mathcal{W}_i(x_t) - \mathbbm{1}_k}{\gamma} - (2y_t - \mathbbm{1}_k) \right)\right] \\
        &= \frac{1}{\gamma}\sum_{i=1}^{N} \sum_{t=1}^{T}\E\left[p_t^i \cdot 2(\mathcal{W}_i(x_t) - \mathbbm{1}_k)\right]   - \sum_{i=1}^{N} \sum_{t=1}^{T} \E\left[p_t^i \cdot (2y_t - \mathbbm{1}_k)\right] \\
        &= \frac{1}{\gamma}\sum_{i=1}^{N} \sum_{t=1}^{T}\E\left[ 2\mathcal{W}_i(x_t) \cdot y_t^i - 1)\right]   - \sum_{i=1}^{N} \sum_{t=1}^{T} \E\left[p_t^i \cdot (2y_t - \mathbbm{1}_k)\right] && \text{(Lemma \ref{expectation})}\\
        &= \frac{1}{\gamma}\sum_{i=1}^{N} \sum_{\ell \in \mathcal{B}_k} \sum_{t: y_t = \ell}\E\left[ 2\mathcal{W}_i(x_t) \cdot y_t^i - 1)\right]   - \sum_{i=1}^{N} \sum_{t=1}^{T} \E\left[p_t^i \cdot (2y_t - \mathbbm{1}_k)\right]\\
    \end{align*}
    
    Using the weak learning condition in Definition \ref{WLC3},
    
    \begin{align*}
         \frac{1}{\gamma}\sum_{i=1}^{N} \sum_{\ell \in \mathcal{B}_k} \sum_{t: y_t = \ell}\E\left[ 2\mathcal{W}_i(x_t) \cdot y_t^i - 1)\right] &\geq  \sum_{i=1}^{N} \sum_{\ell \in \mathcal{B}_k} \sum_{t: y_t = \ell}\E\left[ 2h^*_{i\ell}(x_t) \cdot y_t^i - 1)\right] - \frac{kNR_{\mathcal{W}}(T)}{\gamma} && \text{(Definition \ref{WLC2})}\\
         &\geq \sum_{i=1}^{N} \sum_{\ell \in \mathcal{B}_k} \sum_{t: y_t = \ell}\E\left[ 2h^*(x_t) \cdot y_t^i - 1)\right] - \frac{kNR_{\mathcal{W}}(T)}{\gamma}\\
         &= \sum_{i=1}^{N} \sum_{t = 1}^T\E\left[ 2h^*(x_t) \cdot y_t^i - 1)\right] - \frac{kNR_{\mathcal{W}}(T)}{\gamma}\\
         &= \sum_{i=1}^{N}\sum_{t=1}^{T}\E\left[ p_t^i \cdot (2h^*(x_t) - \mathbbm{1}_k)\right] - \frac{kNR_{\mathcal{W}}(T)}{\gamma} && \text{(Lemma \ref{expectation})}\\
    \end{align*}
    
    Putting things together, we find, 
    \begin{align*}
        \E\left[\sum_{i=1}^{N} \sum_{t=1}^{T} l_t^i(p_t^i)\right] &\geq \sum_{i=1}^{N}\sum_{t=1}^{T}\E\left[ p_t^i \cdot (2h^*(x_t) - 2y_t) \right] - \frac{kNR_{\mathcal{W}}(T, k)}{\gamma}\\
        &\geq \sum_{i=1}^{N}\sum_{t=1}^{T}\E\left[ 2(h^*(x_t) \cdot y_t - 1) \right] - \frac{kNR_{\mathcal{W}}(T, k)}{\gamma} && \text{(Lemma \ref{lowerbound})}\\
        &= N\sum_{t=1}^{T} 2(h^*(x_t) \cdot y_t - 1)  - \frac{kNR_{\mathcal{W}}(T, k)}{\gamma}.\\
    \end{align*}
    
    Combining this lower bound with the upper bound on the expected sum of losses in the proof of Theorem \ref{thm1} completes this proof. 
\end{proof}

We end this section by presenting one more online agnostic weak learning condition that places an \textit{asymmetric} gain function on the best competitor and the weak learner. One of the benefits of this condition is that as $k$ increases, it is more explicit how exactly the weak learning assumption becomes stronger.

 \begin{definition}
 \label{WLC4}
  Let $\mathcal{H} \subseteq \mathcal{B}_k^{\mathcal{X}}$ be a class of experts and let $0 < \gamma \leq 1$ denote the \say{advantage}. An online learning algorithm $\mathcal{W}$ is a $(\gamma, T)$-agnostic weak online learner for $\mathcal{H}$ if for any adaptively chosen sequence of tuples $(x_t, y_t) \in \mathcal{X} \times \mathcal{B}_k$, at every iteration $t \in [T]$, the algorithm outputs $\mathcal{W}(x_t) \in \mathcal{B}_k$ such that, 
 
  \[\gamma \max_{h \in \mathcal{H}} \E\left[ \sum_{t = 1}^T \frac{kh(x_t)\cdot y_t - 1}{k-1} \right] - \E\left[ \sum_{t=1}^T 2\mathcal{W}(x_t) \cdot y_t - 1\right] \leq R_\mathcal{W}(T, k),\]
  
 where the expectation is taken w.r.t. the randomness of the weak learner $\mathcal{W}$ and that of the possibly adaptive
adversary, $R_\mathcal{W} : \mathbbm{N} \times \mathbbm{N} \rightarrow \mathbbm{R}_+$ is the additive regret: a non-decreasing, sub-linear function of $T$.
 \end{definition}
  
Along the same lines as above, under Definition \ref{WLC4}, one can show that Algorithm \ref{OnlineAgnosticBoostFull} achieves the following average regret bound. 

  \begin{proposition}
   Assuming weak learners satisfy Definition \ref{WLC4}, the expected regret bound of Algorithm \ref{OnlineAgnosticBoostFull} is
   
    \[\frac{1}{T}\max_{h \in \mathcal{H}}\E\left[ \sum_{t=1}^{T} 2h(x_t) \cdot y_t - 1 \right] - \frac{1}{T}\E\left[\sum_{t=1}^{T}2\hat{y}_t \cdot y_t - 1\right] \leq \frac{R_{\mathcal{W}}(T, k)}{\gamma T} + \frac{R_{\mathcal{A}}(N)}{N}.\]

  \end{proposition}

 \begin{proof}
Again, we only need to show a lower bound on the expected sum of losses passed to the OCO algorithm. Define  $h^* = \argmax_{h \in \mathcal{H}}\sum_{t=1}^T  (2h(x_t)\cdot y_t - 1)$ as the optimal competitor in hindsight. Then, 
      
      \begin{align*}
        \E\left[\sum_{i=1}^{N} \sum_{t=1}^{T} l_t^i(p_t^i)\right] &= \E\left[\sum_{i=1}^{N} \sum_{t=1}^{T} p_t^i \cdot \left(\frac{2\mathcal{W}_i(x_t) - \mathbbm{1}_k}{\gamma} - (2y_t - \mathbbm{1}_k) \right)\right] \\
        &= \frac{1}{\gamma}\sum_{i=1}^{N} \sum_{t=1}^{T}\E\left[p_t^i \cdot (2\mathcal{W}_i(x_t) - \mathbbm{1}_k)\right]   - \sum_{i=1}^{N} \sum_{t=1}^{T} \E\left[p_t^i \cdot (2y_t - \mathbbm{1}_k)\right] \\
        &= \frac{1}{\gamma}\sum_{i=1}^{N} \sum_{t=1}^{T}\E\left[ 2\mathcal{W}_i(x_t) \cdot y_t^i - 1\right]   - \sum_{i=1}^{N} \sum_{t=1}^{T} \E\left[p_t^i \cdot (2y_t - \mathbbm{1}_k)\right] && \text{(Lemma \ref{expectation})}\\
    \end{align*}
    
Using the weak learning condition in Definition \ref{WLC4},

    \begin{align*}
       \frac{1}{\gamma}\sum_{i=1}^{N} \sum_{t=1}^{T}\E\left[ 2\mathcal{W}_i(x_t) \cdot y_t^i - 1\right] &\geq \sum_{i=1}^{N} \max_{h \in \mathcal{H}}\sum_{t=1}^{T}\E\left[ \frac{kh(x_t) \cdot y_t^i - 1}{k-1}\right] - \frac{NR_{\mathcal{W}}(T, k)}{\gamma} && \text{(Definition \ref{WLC4})}\\
        &\geq \sum_{i=1}^{N} \sum_{t = 1}^T\E\left[ 2h^*(x_t) \cdot y_t^i - 1)\right]  - \frac{NR_{\mathcal{W}}(T, k)}{\gamma}\\
        &= \sum_{i=1}^{N}\sum_{t=1}^{T}\E\left[ p_t^i \cdot (2h^*(x_t) - \mathbbm{1}_k)\right] - \frac{NR_{\mathcal{W}}(T, k)}{\gamma}. && \text{(Lemma \ref{expectation})}
    \end{align*}
    
Putting things together, we find, 
    \begin{align*}
        \E\left[\sum_{i=1}^{N} \sum_{t=1}^{T} l_t^i(p_t^i)\right] &\geq \sum_{i=1}^{N}\sum_{t=1}^{T}\E\left[ p_t^i \cdot (2h^*(x_t) - 2y_t) \right] - \frac{NR_{\mathcal{W}}(T, k)}{\gamma}\\
        &\geq \sum_{i=1}^{N}\sum_{t=1}^{T}\E\left[ 2(h^*(x_t) \cdot y_t - 1) \right] - \frac{NR_{\mathcal{W}}(T, k)}{\gamma} && \text{(Lemma \ref{lowerbound})}\\
        &= N\sum_{t=1}^{T} 2(h^*(x_t) \cdot y_t - 1)  - \frac{NR_{\mathcal{W}}(T, k)}{\gamma}.\\
    \end{align*}

Note that this is the same lower bound as in the proof of Theorem \ref{thm1} and therefore the same regret bound as in Theorem \ref{thm1} holds.     
 \end{proof}

\section{Dependence of $R_{\mathcal{W}}(T, k)$ on $k$}
\label{deponk}

We show concretely how the weak learner's regret, $R_\mathcal{W}(T, k)$, can depend on the number of classes $k$. Recall, that the weak learning condition in Definition \ref{WLC} is written in terms of approximately \textit{maximizing} a sequence of \textit{gain} functions:

$$
    \sigma_{y, \ell}(z) = 
    \begin{cases}
      -1, & \text{if}\ z = \ell \\
      1, & \text{if}\ z = y \\
      0, & \text{otherwise}
    \end{cases}
$$

with an advantage parameter $0 < \gamma < 1$. By taking an affine transformation of  $ \sigma_{y, \ell}(z)$, Definition \ref{WLC} can be equivalently expressed as approximately \textit{minimizing} a sequence of bounded, non-negative \textit{loss} functions

$$L_{y, l}(z) =  \frac{1- \sigma_{y, \ell}(z)}{2} = \begin{cases}
      1, & \text{if}\ z = \ell \\
      0, & \text{if}\ z = y \\
      1/2, & \text{otherwise}
    \end{cases}
$$

with an advantage parameter $\gamma > 1$.  By redefining weak learning in terms of bounded, non-negative loss functions,  we can tap into the rich literature of Prediction with Expert Advice to construct online agnostic weak learners. Specifically, we will use the celebrated (Randomized) Exponential Weights Algorithm (EWA). A nice fact about the EWA is that for 0-1 losses $\tilde{L}_{y_t}(z_t) = \mathbbm{1}\{z_t \neq y_t\}$, and any \textit{finite} set of experts $\mathcal{H}$, it enjoys the regret bound

$$\E\left[\sum_{t=1}^T \tilde{L}_{y_t}(\hat{y}_t) \right]  \leq \frac{\eta}{1-e^{-\eta}}\E\left[\min_{h \in \mathcal{H}}\sum_{t=1}^T \tilde{L}_{y_t}(h(x_t)) \right] + \frac{\ln(|\mathcal{H}|)}{1-e^{-\eta}},$$

where $\hat{y}_t$ denotes the prediction of the EWA in the $t$'th iteration and $\eta > 0$ is a tuneable learning rate (see the the exponentially-weighted forecaster from Chapter 2 of \cite{10.5555/1137817}). We will be using this regret bound extensively in the next two subsections.

\subsection{Finite Hypothesis Classes}

In this section, we will consider two types of finite hypothesis classes: the set of discretized weight matrices and the set of discretized multiclass decision trees of depth $1$. 

Beginning with weight matrices, consider the multiclass classification setup with example label pairs $(x, y) \in \mathbbm{R}^d \times \mathcal{B}_k$. Let $\mathcal{H}$ denote a finite hypothesis class parameterized by discretized weight matrices $W_h \in \mathbbm{R}^{k \times d} $, such that $|\mathcal{H}| = \left(\frac{1}{\delta}\right)^{kd}$, where $\delta \in (0, 1)$ is the level of discretization. For an input example $x \in \mathbbm{R}^d$, a hypothesis $h \in \mathcal{H}$ makes its prediction as $\hat{y_t} = \argmax_k W_hx.$ Take $\mathcal{H}$ to be a set of experts. We will now show that an instance of the EWA, using $\mathcal{H}$ and learning rate $\eta$, run over 0-1 losses corresponds to an agnostic weak online learner with advantage parameter $\gamma = \frac{2\eta}{1 - e^{-\eta}}$ and regret $R_\mathcal{W}(T, k) = O(kd\log(\frac{1}{\delta}))$. Taking the discretization parameter $\delta = \frac{1}{T}$ ensures that the regret remains sublinear in $T$.
 
First, recall that for 0-1 losses $\tilde{L}_{y_t}(z_t) = \mathbbm{1}\{z_t \neq y_t\}$, the EWA algorithm $\mathcal{W}$ using hypothesis class $\mathcal{H}$ and learning rate $\eta > 0$ guarantees the regret bound

$$\E\left[\sum_{t=1}^T \tilde{L}_{y_t}(\mathcal{W}(x_t)) \right]  \leq \frac{\eta}{1-e^{-\eta}}\E\left[\min_{h \in \mathcal{H}}\sum_{t=1}^T \tilde{L}_{y_t}(h(x_t)) \right] + O(\log|\mathcal{H}|).$$
Next, for any values of $\ell_1, .., \ell_T$ where $\ell_t \neq y_t$, observe that 

$$\E\left[\sum_{t=1}^T L_{y_t, l_t}(\mathcal{W}(x_t)) \right] \leq \E\left[\sum_{t=1}^T \tilde{L}_{y_t}(\mathcal{W}(x_t)) \right]$$

and 

$$\frac{1}{2}\E\left[\min_{h \in \mathcal{H}}\sum_{t=1}^T \tilde{L}_{y_t}(h(x_t)) \right] \leq \E\left[\min_{h \in \mathcal{H}}\sum_{t=1}^T L_{y_t, \ell_t}(h(x_t)) \right].$$

Putting things together, we get

$$\E\left[\sum_{t=1}^T L_{y_t, l_t}(\mathcal{W}(x_t)) \right]  \leq \frac{2\eta}{1-e^{-\eta}}\E\left[\min_{h \in \mathcal{H}}\sum_{t=1}^T L_{y_t, \ell_t}(h(x_t)) \right] + O(\log |\mathcal{H}| ).$$

Finally, observing that $|\mathcal{H}| = \left(\frac{1}{\delta}\right)^{kd}$ completes the proof. As an example, if one takes $\eta = 1$, then the EWA algorithm corresponds to an agnostic weak online learner with advantage parameter $\gamma = \frac{2}{1 - e^{-1}} \approx 3.16$ and regret $R_\mathcal{W}(T, k) = O(kd\log T)$.

A similar procedure can be performed when considering the same multiclass classification setup when the hypothesis class $\mathcal{H}$ is the set of depth $1$  multiclass decision trees.  If one further restricts to a $k$-wise split on a feature $j \in [d]$, then the class of depth 1 decision trees can be abstractly represented by the function: 

$$f(x) = \begin{cases}
       y_1, & \text{if}\ -\infty < x_j \leq \tau_1 \\
       y_2, & \text{if}\ \tau_1 < x_j \leq \tau_2 \\
       ... \\
       y_k, & \text{if}\ \tau_{k-1} < x_j < \infty \\
    \end{cases}
$$

for input $x \in \mathbbm{R}^d$,  thresholds $\tau_1, ..., \tau_{k-1}$, and labels $y_1, ..., y_k$. If one discretizes the thresholds using the parameter $\delta \in (0, 1)$, then  $|\mathcal{H}| = O(d\cdot k^k \cdot (\frac{1}{\delta})^k)$. Therefore, EWA using this hypothesis class and 0-1 losses corresponds to an agnostic weak online learner with regret $R_\mathcal{W}(T, k) = O(k\log(kT) + \log d)$.

\subsection{Infinite Hypothesis Classes}
Our construction of weak learners for the two learning settings above crucially relied on the fact that the hypothesis class was finite. Below, we discuss how to construct agnostic weak online learners for infinite hypothesis classes. The key insight from the constructions above is that any agnostic weak online learner with advantage $\gamma > 1$ for the standard 0-1 loss can be converted into weak learner that satisfies an equivalent version of Definition \ref{WLC} with the non-negative loss functions $L_{y, l}(z)$ with twice the advantage, namely $2\gamma$. For finite hypothesis classes, we constructed weak learners for the 0-1 loss by running (randomized) EWA over $\mathcal{H}$ with a fixed learning rate $\eta$. Similarly, for infinite hypothesis classes, we can construct a weak learner for 0-1 losses by setting the learning rate $\eta$ to be a fixed positive constant in the optimal multiclass agnostic online learner proposed in \citet{daniely2015multiclass}. Concretely, the multiclass agnostic online learner proposed in Section 5 of \citet{daniely2015multiclass} runs the (randomized) EWA over a carefully constructed finite set of experts of size at most $(Tk)^{LDim(\mathcal{H})}$, where $LDim(\mathcal{H})$ is the Multiclass Littlestone Dimension of $\mathcal{H}$. Following an identical analysis as in \citet{daniely2015multiclass}, one can now show that if $\mathcal{W}$ is a EWA algorithm, then for 0-1 losses over this carefully constructed set of experts, $\mathcal{W}$ achieves regret

$$\E\left[\sum_{t=1}^T \tilde{L}_{y_t}(\mathcal{W}(x_t)) \right]  \leq \frac{\eta}{1-e^{-\eta}}\E\left[\min_{h \in \mathcal{H}}\sum_{t=1}^T \tilde{L}_{y_t}(h(x_t)) \right] + O(LDim(\mathcal{H})\log(Tk)),$$

where $\mathcal{H}$ is the hypothesis class of interest and $\eta > 0$. Picking $\eta = 1$, we get that the EWA $\mathcal{W}$ is an agnostic online weak learner for 0-1 losses with  advantage parameter $ \approx 1.58$. Thus, using the key insight mentioned above, $\mathcal{W}$ is also a agnostic online weak learner for the non-negative loss functions $L_{y, l}(z)$ with advantage parameter $\approx 3.16$. This suffices to show that $\mathcal{W}$ is also an agnostic online weak learner with respect to Definition \ref{WLC}.

\section{Important Lemmas}
\label{AppendixLemmas}
We give the important lemmas (and their proofs) used in the main text and Appendix \ref{AppendixStatAgnBoost} and \ref{AppendixRealBoost}. 

\begin{lemma} 
\label{expectation}

For any $t$ and $i$,
$\E\left[y_t^i\cdot \mathcal{W}_i(x_t)\right] = \E\left[p_t^i\cdot \mathcal{W}_i(x_t)\right]$. 

\begin{proof}
First note that $\E\left[y_t^i|p_t^i, y_t\right] = p_t^i$. Next note that $\mathcal{W}_i(x_t)$ and $y_t^i$ are conditionally independent given $p_t^i$ and $y_t$. Then, 

    \begin{align*}
        \E\left[y_t^i \cdot \mathcal{W}_i(x_t)\right] &= \E_{p^i_t, y_t}\left[\E\left[y_i^t \cdot \mathcal{W}_i(x_t)|p_t^i, y_t\right]\right] && \text{(law of total expectation)} \\
        &= \E_{p^i_t, y_t}\left[\E\left[y_t^i|p_t^i, y_t\right] \cdot \E\left[ \mathcal{W}_i(x_t)|p_t^i, y_t\right] \right] \\
        &= \E_{p^i_t, y_t}\left[p_t^i \cdot \E\left[ \mathcal{W}_i(x_t)|p_t^i, y_t\right] \right] \\
        &= \E_{p^i_t, y_t}\left[\E\left[p_t^i \cdot \mathcal{W}_i(x_t)|p_t^i, y_t\right] \right ]\\
        &= \E\left[p_t^i \cdot \mathcal{W}_i(x_t)\right]. 
    \end{align*}

\end{proof}
\end{lemma}

\begin{lemma}
\label{lowerbound}
For any pair $h, y \in \mathcal{B}_k$ and $p \in \Delta_k$, 

\[p \cdot (2h - 2y) \geq 2 (h\cdot y - 1)\]

\begin{proof}
Consider the case when $h = y$. Then the left and right side of the above inequality hold with equality at the value of $0$ for any vector $p$ in the simplex. Now, consider the case when $h \neq y$. Then, observe that setting $p = y$ achieves the minimum of $p \cdot (2h - 2y)$ at $-2$ which matches the right-hand side. 
\end{proof}
\end{lemma}

\begin{lemma}
\label{upperbound}
For every $0 < \gamma \leq 1$, $h \in \Delta_k$, and $y \in \mathcal{B}_k$, there exists $p \in \Delta_k$ such that:

\[p\cdot \left(\frac{2h-\mathbbm{1}_k}{\gamma} - (2y - \mathbbm{1}_k)\right) \leq 2\left(\prod\left(\frac{h}{\gamma}\right)\cdot y - 1\right),\]

where $\prod: \mathbbm{R}^k \rightarrow \Delta_k$ is the $L_2$ projection operator onto the simplex. 
\end{lemma}

\begin{proof} We first review an important property of the $L_2$ projection onto the simplex. It is known that the solution to the $L_2$ projection $\prod (\cdot)$ onto the simplex is just a threshold operator (see Theorem 2.2 in \cite{chen2011projection}): for any input vector $g$, subtract a constant $\mu$ from each element. If the result is negative, replace it by zero. More formally, $\prod (g) = \left[g - \mu \mathbbm{1}_k\right]_{+}.$ Note that $\mu$ must satisfy the piecewise linear equation $\sum_{i=1}^k \max(0, g_i - \mu) = 1.$ Let $g^{'}$ denote the sorted version of $g$ in decreasing order. Then, define $K$ as the largest integer within $\{1, 2, ..., k\}$ such that $g^{'}_K - \frac{\sum_{j=1}^{K}g^{'}_j - 1}{K} > 0$. One can verify that $\mu = \frac{\sum_{j=1}^Kg^{'}_j - 1}{K}$ is the unique solution to the piece-wise linear equation above. We will use this property of $L_2$ projection extensively in our proof, which we formally begin below. 

Define $\tilde{h} = \prod\left(\frac{h}{\gamma}\right)$. 
Let $t$ be the number of zero entries in the projection $\tilde{h}$. Furthermore, define $h_y = h \cdot y$ and $\tilde{h}_y = \tilde{h} \cdot y$ as the $y$'th index of $h$ and $\prod(\frac{h}{\gamma})$ respectively. We will show that when $t \leq k - 2$, setting $p = y$ achieves the desired bound and when $t = k-1$, there exists some $p \in \Delta_k$ that achieves the desired bound.

We now begin with the case where $t \leq k - 2$. Our goal will be to first show the following lower bound on $\tilde{h}_y$,

$$\tilde{h}_y \geq \frac{(k-t)h_y - 1}{\gamma(k - t)} + \frac{1}{k-t}.$$

Consider the subcase where $\tilde{h}_y = 0$. Then, by the properties of the projection operator above it must have been the case that $\frac{h_y}{\gamma}  \leq \mu,$ where using the definition above $\mu = \frac{\frac{1}{\gamma}\sum_{j=1}^{k-t} h^{'}_j - 1}{k-t}.$ Therefore, $h_y \leq \frac{\sum_{j=1}^{k-t} h^{'}_j - \gamma}{k-t} \leq \frac{1 - \gamma}{k-t}.$ Substituting in, we find

$$\frac{(k-t)h_y - 1}{\gamma(k - t)} + \frac{1}{k-t} \leq 0 = \tilde{h}_y.$$

Now, consider the subcase where $\tilde{h}_y > 0$. Then, $\tilde{h}_y = \frac{h_y}{\gamma} - \mu,$ where again

$$\mu = \frac{\frac{1}{\gamma}\sum_{j=1}^{k-t} h^{'}_j - 1}{k-t} \leq \frac{1-\gamma}{\gamma(k - t)}.$$

Substituting in completes proving the lowerbound 

$$\tilde{h}_y = \frac{h_y}{\gamma} - \mu \geq \frac{h_y}{\gamma} - \frac{1-\gamma}{\gamma(k - t)} = \frac{(k-t)h_y - 1}{\gamma(k - t)} + \frac{1}{k-t}.$$

Now that we have shown, 

$$\tilde{h}_y = \frac{h_y}{\gamma} - \mu \geq \frac{(k-t)h_y - 1}{\gamma(k - t)} + \frac{1}{k-t},$$

we finally show this implies the inequality

$$\left(\frac{2h_y - 1}{\gamma} - 1\right) \leq 2\tilde{h}_y - 2,$$

which follows from plugging $p = y$ into the lemma. We start with 

\begin{align*}
    2\tilde{h}_y - \frac{2h_y -1}{\gamma} - 1 &\geq \frac{2(k-t)h_y - 2}{\gamma(k - t)} + \frac{2}{k-t} - \frac{2h_y -1}{\gamma} - 1\\
    &=  -\frac{2}{\gamma(k-t)} + \frac{1}{\gamma} + \frac{2}{k-t} - 1\\
    &= (\frac{1}{\gamma} - 1) - \frac{2}{k-t}(\frac{1}{\gamma} - 1)\\
    &= (\frac{1}{\gamma} - 1)(1 - \frac{2}{k-t}) \geq 0,
\end{align*}

where the last inequality follows from the assumption that $ t \leq k - 2$. Thus, for the case where $ t \leq k - 2$, we have shown the lemma holds by setting $p = y$.

Now we will consider the case where $t = k - 1$. Again, we will also consider the subcases where $\tilde{h}_y = 0$ and $\tilde{h}_y = 1$. We start with the subcase where $\tilde{h}_y = 0$. Here, we will show that picking $p = y$ is the right choice. Namely,  we will show that the following inequality holds

$$\left(\frac{2h_y - 1}{\gamma} - 1\right) \leq 2\tilde{h}_y - 2.$$

Under the assumption that $\tilde{h}_y = 0$, the right hand side collapses to $-2$. Thus, we need to show that $\frac{2h_y - 1}{\gamma} \leq -1.$ Recall that if $\tilde{h}_y = 0$, then by projection properties. $\frac{h_y}{\gamma}  \leq \mu.$ Again, by definition, 

$$\mu = \frac{\frac{1}{\gamma}\sum_{j=1}^{k-t} h^{'}_j - 1}{k-t} = \max_{j}\frac{h_j}{\gamma} - 1 \leq \frac{1 - h_y}{\gamma} - 1.$$

Therefore we find that, $\frac{h_y}{\gamma} \leq \frac{1 - h_y}{\gamma} - 1,$ from which it is easy to see that $\frac{2h_y - 1}{\gamma} \leq -1,$ which completes this subcase. 

We now move to the subcase where $\tilde{h}_y = 1$. Here, we will show that there exists a $p \in \mathcal{B}_k \setminus \{y\}$ that satisfies the required bound in the lemma. When $\tilde{h}_y = 1$, the right hand side collapses to $0$. Thus, we need to show the existence of $p \in \mathcal{B}_k \setminus \{y\}$ s.t. 

$$\frac{2h_p - 1}{\gamma} + 1 \leq 0,$$

where $h_p$ denotes $h \cdot p$, the value of $h$ at the $p$'th index. We shall prove the bound above via contradiction. Assume that indeed for all $p \in \mathcal{B}_k \setminus \{y\}$,

$$\frac{2h_p - 1}{\gamma} + 1 > 0.$$

This implies that  $h_p > \frac{1-\gamma}{2}$ for all $p \in \mathcal{B}_k \setminus \{y\}$. Observe that when $k \geq 3$, the probability mass over all labels other than $y$ is \textit{strictly} bounded below by $1-\gamma$ and so it must be that $h_y  < \gamma$.  Recall that if $\tilde{h}_y = 1$, then by projection properties, 

$$\mu = \frac{\frac{1}{\gamma}\sum_{j=1}^{k-t} h^{'}_j - 1}{k-t} = \frac{h_y}{\gamma} - 1 < 0,$$

where the last inequality follows from the fact that $h_y$ must have mass strictly less than $\gamma$. However, if $\mu < 0$, then since the solution to the projection operation is $\left[\frac{h}{\gamma} - \mu \mathbbm{1}_k\right]_{+}$ and the entries of $h$ are non-negative, all entries of $\frac{h}{\gamma} - \mu \mathbbm{1}_k$ are strictly positive. This contradicts our original assumption that $t = k - 1$ (implying that all but one index are negative) which completes the proof for $k \geq 3$. 

Now, when $k = 2$, $\mu = \frac{h_y}{\gamma} - 1$, and by projection properties, 

$$\frac{h_p}{\gamma} - \mu = \frac{h_p - h_y}{\gamma} + 1 \leq 0.$$

Noting that $h_y = 1 - h_p$, completes the proof since it implies that for $p \neq y$, $\frac{2h_p - 1}{\gamma} + 1 \leq 0.$
\end{proof}

\begin{lemma}
\label{upperbound_real}
For every $0 < \gamma \leq 1$, $h \in \Delta_k$, and $y \in \mathcal{B}_k$, there exists $p \in [0, 1]$ such that:

\[p\left(\frac{2h\cdot y-1}{\gamma} - 1\right) \leq 2\left(\prod\left(\frac{h}{\gamma}\right)\cdot y - 1\right),\]

where $\prod: \mathbbm{R}^k \rightarrow \Delta_k$ is the $L_2$ projection operator onto the simplex.
\end{lemma}

\begin{proof}
   Observe that we can consider the same four cases as in the proof for Lemma \ref{upperbound}. Indeed, for three out of the four cases, the optimal choice for Lemma \ref{upperbound} was selecting $p = y$. Under these three cases, we know that by picking $p = y$ the following inequality holds by substituting $p = y$ into Lemma  \ref{upperbound}:
   
   $$\frac{2h\cdot y-1}{\gamma} - 1 \leq 2\left(\prod\left(\frac{h}{\gamma}\right)\cdot y - 1\right).$$
   
   Thus, for this lemma, for the same three cases of Lemma \ref{upperbound}, it suffices to pick $p = 1$. In the last case of Lemma \ref{upperbound}, $\tilde{h}_y = 1$, that is, all the mass after the projection falls on the label $y$. In this case, the right hand side of our inequality collapses to $0$. Thus it suffices to pick $p = 0$ for this lemma to hold.

\end{proof}




\begin{lemma}
\label{equivalence}
For $0 < \gamma \leq 1$, $h \in \Delta_k$, $y \in \mathcal{B}_k$, and all $p \in \Delta_k$:

$$p \cdot \left(\frac{2h - \mathbbm{1}_k}{\gamma} - (2y - \mathbbm{1}_k)\right) = (2p - \mathbbm{1}_k)\cdot (\frac{h}{\gamma} - y)$$

\end{lemma}

\begin{proof}
    \begin{align*}
        p \cdot \left(\frac{2h - \mathbbm{1}_k}{\gamma} - (2y - \mathbbm{1}_k)\right) &= p \cdot \left(\frac{2h - \mathbbm{1}_k}{\gamma} \right) - p \cdot \left(2y - \mathbbm{1}_k\right)\\
        &= 2p \cdot \frac{h}{\gamma} - \frac{1}{\gamma} - 2p \cdot y + 1\\
        &= 2p \cdot (\frac{h}{\gamma} - y) - \frac{1}{\gamma} + 1\\
        &= 2p \cdot (\frac{h}{\gamma} - y) - \mathbbm{1}_k \cdot (\frac{h}{\gamma} - y)\\
        &= (2p - \mathbbm{1}_k)\cdot (\frac{h}{\gamma} - y).
    \end{align*}
\end{proof}
\section{Experimental Details}
\label{ExpDet}
Each cell in Table \ref{tab:data} is the average over the "best accuracies" of five independent shuffles of the dataset. For each shuffle, the "best accuracy" was computed as the maximum accuracy over five candidate $\gamma$ values.  For both OCO-based boosting algorithms, $\gamma$ was tuned across $[0.10, 0.30, 0.50, 0.70, 1]$, and for OnlineMBBM, $\gamma$ was tuned across $[0.001, 0.01, 0.05, 0.1, 0.3]$. The range of $\gamma$ values used to tune OnlineMBBM is consistent with those used by \citet{jung2017online}. In addition, in our own experiments, we found that for $\gamma$ values larger than 0.30, OnlineMBBM performed significantly worse. Note that AdaBoost.OLM does have not have a tuning parameter $\gamma$, which is one of its advantages despite its subpar performance. 

For both OCO-based boosting algorithms, we used Online Gradient Descent with learning rate $\eta = \frac{\gamma}{\sqrt{N}}$ as the OCO algorithm. This learning rate is optimal up to constant factors \cite{zinkevich2003online}.All computations were carried out on a Nehalem architecture 10-core 2.27 GHz Intel Xeon E7-4860 processors with 25 GB RAM per core. The total amount of computing time was around 500 hours. The River package \cite{2020river}, used to implement the VeryFastDecisionTree, is licensed under BSD 3-Clause. 

 Table \ref{tab:dataset_summary} provides more information about each of the datasets used. For the Mice dataset, entries with missing data were replaced with the average value of their respective column.

\begin{table}
\caption{Dataset summaries.}
\begin{center}
\begin{tabular}{| c | c | c | c |}
\hline
\textbf{Dataset} & \textbf{\#datapoints} & \textbf{\#covariates} & \textbf{\#classes} \\
\hline

Balance & 625 & 4 & 3 \\
Cars & 1728 & 6  & 4\\ 
LandSat & 6435  & 36 & 6\\ 
Segmentation & 2310 & 19 & 7\\ 
Mice & 1080 & 82  & 8 \\ 
Yeast & 1483 & 8 & 10\\ 
Abalone & 4177 & 8  & 28\\

\hline
\end{tabular}
\end{center}
\label{tab:dataset_summary}
\end{table}

Table \ref{tab:SE} provides the standard error for each accuracy in Table \ref{tab:data}. 

\begin{table}
\caption{Standard errors of accuracies.}
\begin{center}
\begin{tabular}{| c | c | c | c | c |}
\hline
\textbf{Dataset} & \multicolumn{4}{ c |}{\textbf{Standard Errors of Accuracies}}  \\ 
\cline{2-5}
& Agn & Opt & Ada & OCOR \\
\hline

Balance & 0.68 & 1.06 & 1.27 & 0.84 \\
Cars & 0.46 & 0.02  & 0.55  & 0.32\\ 
LandSat & 0.89  & 0.36 & 1.23 & 0.21\\ 
Segmentation & 0.70 & 0.48 & 1.56 & 0.54\\ 
Mice & 0.58 & 0.60  & 2.53 & 0.26 \\ 
Yeast & 0.43 & 1.03 & 1.09 & 1.53\\ 
Abalone & 0.53 & 0.42  & 0.28 & 0.31\\

\hline
\end{tabular}
\end{center}
\label{tab:SE}
\end{table}

\end{document}